\documentclass[10pt,journal]{IEEEtran}

\usepackage{mathtools,amsmath,amsthm}

\usepackage[utf8]{inputenc} 
\usepackage[T1]{fontenc}    
\usepackage{hyperref}       
\usepackage{url}            
\usepackage{booktabs}       
\usepackage{amsfonts}       
\usepackage{nicefrac}       
\usepackage{microtype}      
\usepackage{caption}
\usepackage{subcaption}
\usepackage{xcolor}       
\usepackage{cleveref}
\usepackage[binary-units]{siunitx}
\let\labelindent\relax

\usepackage{enumitem}
\usepackage{mathtools}
\usepackage{multirow}
\usepackage{amsthm}
\usepackage{algorithm}
\usepackage{algpseudocode}

\theoremstyle{plain}
\newtheorem{Theorem}{Theorem}
\newtheorem{Lemma}{Lemma}
\newtheorem{Definition}{Definition}

\newtheorem{Problem}{Problem}
\newtheorem{Remark}{Remark}

\newcommand{\pinom}{\pi_\mathrm{nom}}
\newcommand{\piqp}{\pi_\mathrm{QP}}

\newcommand{\nrays}{n_{\mathrm{rays}}}

\newcommand{\abs}[1]{\left\lvert#1\right\rvert}
\newcommand{\norm}[1]{\left\lVert#1\right\rVert}

\DeclareMathOperator{\softmax}{softmax}

\newcommand{\revision}[1]{ #1}

\definecolor{CfBlue}{HTML}{7EB0D5}
\definecolor{CfBlue2}{HTML}{4095d6}

\crefname{figure}{Fig.}{Figs.} \Crefname{figure}{Fig.}{Figs.} \AtBeginEnvironment{appendices}{\crefalias{section}{appendix}}
\AtBeginEnvironment{appendices}{\crefalias{subsection}{appendix}}

\allowdisplaybreaks

\IEEEoverridecommandlockouts

\begin{document}

\title{GCBF+: A Neural Graph Control Barrier Function \mbox{Framework for Distributed Safe Multi-Agent Control}}

\author{\mbox{Songyuan Zhang*, Oswin So*, Kunal Garg, and Chuchu Fan}
\thanks{SZ, OS, KG, and CF are with the Department of Aeronautics and Astronautics at MIT, \texttt{\{szhang21,oswinso,kgarg,chuchu\}@mit.edu}.}
\thanks{*Equal contribution}
\thanks{Project website: \href{https://mit-realm.github.io/gcbfplus/}{https://mit-realm.github.io/gcbfplus/}}
}

\maketitle

\begin{abstract}
Distributed, scalable, and safe control of large-scale multi-agent systems is a challenging problem. In this paper, we design a distributed framework for safe multi-agent control in large-scale environments with obstacles, where a large number of agents are required to maintain safety using only local information and reach their goal locations. We introduce a new class of certificates, termed graph control barrier function (GCBF), which are based on the well-established control barrier function theory for safety guarantees and utilize a graph structure for scalable and generalizable distributed control of MAS. We develop a novel theoretical framework to prove the safety of an arbitrary-sized MAS with a single GCBF. We propose a new training framework GCBF+ that uses graph neural networks to parameterize a candidate GCBF and a distributed control policy. The proposed framework is distributed and is capable of taking point clouds from LiDAR, instead of actual state information, for real-world robotic applications. We illustrate the efficacy of the proposed method through various hardware experiments on a swarm of drones with objectives ranging from exchanging positions to docking on a moving target without collision. Additionally, we perform extensive numerical experiments, where the number and density of agents, as well as the number of obstacles, increase. Empirical results show that in complex environments with {agents with nonlinear dynamics} (e.g., Crazyflie drones), GCBF+ outperforms the hand-crafted CBF-based method with the best performance by up to $20\%$ for relatively small-scale MAS { with} up to 256 agents, and leading reinforcement learning (RL) methods by up to $40\%$ for MAS with 1024 agents. Furthermore, the proposed method does not compromise on the performance, in terms of goal reaching, for achieving high safety rates, which is a common trade-off in RL-based methods.
\end{abstract}

\section{Introduction}
\subsection{Background}
Multi-agent systems (MAS) have received tremendous attention from scholars in different disciplines, including computer science and robotics, as a means to solve complex problems by subdividing them into smaller tasks \cite{dorri2018multi}. MAS applications include but are not limited to warehouse operations \cite{baiyu2023DD,kattepur2018distributed}, self-driving cars \cite{schmidt2022introduction,palanisamy2020multi,zhou2021smarts,zhang2023compositional}, coordinated navigation of a swarm of drones in a dense forest for search-and-rescue missions \cite{tian2020search,ghamry2017multiple}; interested reader is referred to  \cite{ju2022review} for an overview of MAS applications.  For such safety-critical MASs, it is important to design controllers that not only guarantee safety in terms of collision and obstacle avoidance but are also scalable to large-scale multi-agent problems. 

\begin{figure}
    \centering
\includegraphics[width=\linewidth]{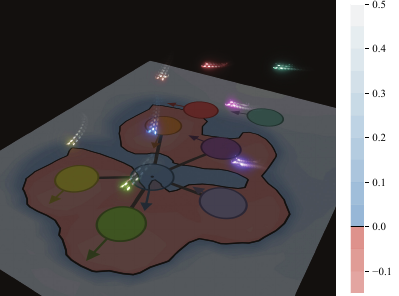}
    \caption{\textbf{8-Crazyflie swapping with GCBF:} We learn a distributed graph control barrier function (GCBF) for an 8-agent swapping task on the Crazyflie hardware platform. We visualize the learned GCBF for the \textcolor{CfBlue2}{\textbf{blue agent}} and draw the edges with its neighboring agents in grey. The learned GCBF can handle arbitrary graph topologies and hence can scale to an arbitrary number of agents {without retraining}.}
    \label{fig:hw_teaser}
\end{figure}

Common MAS motion planning methods include but are not limited to solving mixed integer linear programs (MILP) for computing safe paths for agents \cite{chen2021scalable,afonso2020task} and {sampling-based planning methods such as} rapidly exploring random tree (RRT) \cite{netter2021bounded}. However, such {\textit{centralized-in-execution}}
approaches, where the complete MAS state information is used, are not scalable to large-scale MAS.
The recent work \cite{saravanos2023ADMM}
performing distributed trajectory optimization
proposes a scalable method for MAS control. However, this approach cannot take into account changing neighborhoods and environments, limiting its applications. There has been quite a lot of development in the learning-based methods for MAS control in recent years; see \cite{garg2023learning} for a detailed overview of learning-based methods for safe control of MAS. Multi-agent Reinforcement Learning (MARL)-based approaches, e.g., Multi-agent Proximal Policy Optimization (MAPPO) \cite{yu2022surprising}, have also been adapted to solve multi-agent motion planning problems. However, when it comes to RL, one main challenge for safety, particularly in multi-agent cases, is the tradeoff between the practical performance and the safety requirement because of the conflicting reward-penalty structure \cite{garg2023learning}. We argue that our proposed framework does not suffer from a similar tradeoff and can automatically balance satisfying safety requirements and performance criteria through a carefully designed training loss.  

{In the past few years, control barrier functions (CBFs) have become a popular tool to encode safety requirements for robotic systems \cite{ames2019control}. 
For MAS, safety is generally formulated pair-wise. Therefore, a CBF is assigned for each pair-wise safety constraint, and approximation methods are used to combine multiple constraints \cite{glotfelter2017nonsmooth,jankovic2021collision,cheng2020safe,garg2021robust}. While hand-crafted CBF-based quadratic programs (CBF-QP) have shown promising results for single-agent systems \cite{ames2017control} and \textit{simple} (i.e., linear systems or systems with relative degree one) small-scale multi-agent systems \cite{wang2017safety}, it is difficult to find a CBF when it comes to highly complex nonlinear systems and large-scale MAS.
Another major challenge with such approaches is constructing a CBF in the presence of input constraints, e.g. actuation limits.
There are some recent developments on this topic for MAS, see e.g.,\cite{chen2020guaranteed}; however, as noted by the authors in \cite{agrawal2021safe}, finding a barrier function that satisfies safety conditions in the presence of input bounds is a complex problem.}

\subsection{Our contributions}
To overcome these limitations, in this paper, we first introduce the notion of graph control barrier function, termed GCBF (\Cref{fig:hw_teaser}), for large-scale MAS to address the problems of safety, scalability, and generalizability. The proposed GCBF can account for an arbitrary and changing number of neighbors, and hence, {computed for a small-scale MAS, the same GCBF generalizes to large-scale MAS}. We provide a new safety result using the notion of GCBF that {certifies} the safety of MAS of arbitrary size. This is the first such result that shows the safety of MAS of any size using one barrier function. Next, we introduce GCBF+, a novel { \textit{centralized training distributed execution}} framework to learn a candidate GCBF along with a distributed control policy. We use graph neural networks (GNNs) to capture the changing graphical topology of distance-based observation information flow. 
We propose a novel loss function formulation that accounts for safety, goal-reaching as well as actuation limits, thereby addressing the limitations of hand-crafted CBF-QP methods. Furthermore, the proposed algorithm can work with LiDAR-based point-cloud observations to handle obstacles in real-world environments.  
With these technologies, our proposed framework can generalize well to many challenging settings, including crowded and unseen obstacle environments.

To corroborate the practical applicability of the proposed method, we perform various hardware experiments on a swarm of Crazyflie drones. The hardware experiments consist of the drones safely exchanging positions in a crowded workspace in the presence of {static and moving obstacles} and docking on a moving target while maintaining safety.
We also perform extensive numerical experiments and provide empirical evidence of the improved performance of the proposed GCBF+ framework compared to the prior version of the algorithm (GCBFv0) in \cite{zhang2023distributed}, a state-of-the-art MARL method (InforMARL) \cite{nayak2023scalable}, {MPC method from \cite{sathya2018embedded}}, and hand-crafted CBF-QP methods from \cite{wang2017safety}.
We consider three 2D environments and two 3D environments in our numerical experiments consisting of linear and nonlinear systems. In the obstacle-free case, we train with $8$ agents and test with over $1000$ agents. In the linear cases, the performance improvement (in terms of safety rate) is about $5$\%, while in the nonlinear cases, the performance improvement is more than $30$\%. In the obstacle environment, we consider $8$ obstacles in training, while up to $128$ obstacles are considered in testing. These experiments corroborate that the proposed method outperforms the baseline methods in successfully completing the tasks in various 2D and 3D environments.

\subsection{Differences from conference version}
This paper builds on the conference paper \cite{zhang2023distributed} which presented the GCBFv0 algorithm. 
We propose a new algorithm GCBF+ which improves upon GCBFv0 in the following ways.
\begin{itemize}
    \item \textbf{Algorithmic modifications}: 
In the prior work, the control policy was learned to account only for the safety constraints, and a CBF-based switching mechanism was used for switching between a goal-reaching nominal controller and a safe neural controller. This led to undesirable behavior and deadlocks in certain situations. 
    Furthermore, an online policy refinement mechanism was used in \cite{zhang2023distributed} when the learned controller could not satisfy the safety requirements which required agents to communicate their control actions for the policy update, adding computational overhead. We modify how the training loss is defined so that the safety and goal-reaching requirements do not conflict, making it possible for the training loss to go to zero. In this way, we can use a single controller for both safety and goal-reaching, without an online policy refinement step for higher safety rates. 
\item \textbf{Actuation limits} ~Another major limitation of GCBFv0 is that it does not account for actuator limits which may result in undesirable behavior when implemented on real-world robotic systems.
    In contrast, in the proposed method in this work, the learned controller considers actuator limits through a look-ahead mechanism for approximation of the safe control invariant set. This mechanism ensures that the learned controller satisfies the actuation limits while keeping the system safe.
    \item {\textbf{Theoretical results on generalization} ~While \cite{zhang2023distributed} proves that a GCBF {certifies} safety for a specific size of the MAS, it does not prove that the same GCBF {certifes} safety when the number of agents changes. We advance this theoretical result to 
prove that a GCBF can certify the safety of a MAS of \textit{any} size. This brings the theoretical understanding of the algorithm closer to the empirical results, where we observe our GCBF+ algorithm scaling to over $1000$ agents while being trained with $8$ agents. 
    }
\item \textbf{Hardware experiments} ~We include various hardware experiments on a swarm of Crazyflie drones, thereby demonstrating the practicality of the proposed framework. 
    \item \textbf{Additional numerical experiments} ~We also include various new numerical experiments as compared to the conference version. In particular, we perform experiments with more realistic system dynamics, such as the 6DOF Crazyflie drone, in contrast to simpler dynamics used in the numerical experiments in \cite{zhang2023distributed}. Furthermore, we provide comparisons to new baselines: InforMARL from \cite{nayak2023scalable}, which is a better RL-based method for safe MAS control, {MPC from \cite{sathya2018embedded}}, and centralized and distributed hand-crafted CBF-based methods from \cite{wang2017safety}. 
    \item \textbf{Better performance} ~ We illustrate that the new GCBF+ algorithm proposed in this paper has much better performance than the original GCBFv0 algorithm in the conference version. In particular, in complex environments consisting of {agents with nonlinear dynamics}, GCBFv0 has a success rate of less than $10\%$ while GCBF+ has a success rate of over $55\%$. Furthermore, in crowded 2D obstacle environments, GCBFv0 has a success rate close to $20\%$ while GCBF+ has a success rate of over $95\%$. 
\end{itemize}

\subsection{Related work}
\textbf{Graph-based methods}~ Graph-based planning approaches such as \textit{prioritized} multi-agent path finding (MAPF) \cite{ma2019searching} and conflict-based search for MAPF \cite{sharon2015conflict} can be used for multi-agent path planning for known environments.
However, MAPF does not take into account system dynamics and does not scale to large-scale systems due to computational complexity.
Another line of work for motion planning in obstacle environments is based on the notion of velocity obstacles \cite{arul2021v} defined using collision cones for velocity. Such methods can be used for large-scale systems with safety guarantees under mild assumptions. However, the current frameworks under this notion assume single or double integrator dynamics for agents. The work in \cite{zheng2018magent} scales to large-scale systems, but it only considers a discrete action space and hence does not apply to robotic platforms that use more general continuous input signals. 

{
\textbf{Model predictive control} ~ To tackle MAS, distributed Model predictive control (MPC) methods have been proposed, incorporating multi-agent path planning and machine learning \cite{wang2014synthesis,toumieh2022decentralized,zhu2020trajectory,fedele2023distributed,luis2019trajectory} with distributed optimization \cite{conte2012computational,nedic2018distributed,mestres2023distributed}. Although the computation is distributed, many MPC works require a central node to perform global information exchange. In addition, the computation complexity of MPC methods impedes their scalability. 
}

\textbf{Centralized CBF-based methods} ~ For systems with relatively simple dynamics, such as single integrator, double integrator, and unicycle dynamics, it is possible to use a distance-based CBF \cite{ames2019control}. For systems with polynomial dynamics, it is possible to use the Sum-of-Squares (SoS) \cite{prajna2002introducing} method to compute a CBF \cite{xu2017correctness}. The key idea of SoS is that the CBF conditions consist of a set of inequalities, which can be equivalently expressed as checking whether a polynomial function is SoS. In this manner, a CBF can be computed through convex optimization \cite{xu2017correctness,srinivasan2021extent,zhao2023convex}. However, the SoS-based approaches suffer from the curse of dimensionality (i.e., the computational complexity grows exponentially with the degree of polynomials involved) \cite{ahmadi2016some}. 

\textbf{Distributed CBF-based methods} ~ While centralized CBF is an effective shield for small-scale MAS, due to its poor scalability, it is difficult to use it for large-scale MAS. To address the scalability problem, distributed CBFs have been developed 
{\cite{wang2017safety,zhang2023distributed,cai2021safe,qin2021learning,Fernandez-Ayala2023distributed,wang2024distributed}}. In contrast to centralized CBF where the state of the MAS is used, for a distributed CBF, only the local observations and information available from communication with neighbors are used, reducing the problem dimension significantly. {However, similar to a centralized CBF, it is difficult to hand-craft distributed CBF for agents with nonlinear dynamics and input constraints.  Most of the works on the safety of MAS consider CBF between each pair of agents, but the resulting control set that satisfies all the pair-wise CBF conditions along with input constraints can be empty.}

\textbf{Learning CBFs}~ One way of navigating the challenge of hand-crafting a CBF is to use neural networks (NNs) for learning a CBF \cite{dawson2023safe}. In the past few years, machine learning (ML)-based methods have been used to learn CBFs for complex systems \cite{qin2021learning,dawson2022safe, qin2022sablas,so2023train,meng2021reactive}
. However, it is challenging for them to balance safety and task performance for multi-task problems, and some methods are not scalable to large-scale multi-agent problems.
The Multi-agent Decentralized CBF (MDCBF) framework in \cite{qin2021learning} uses an NN-based CBF designed for MAS. However, they do not encode a method of distinguishing between other controlled agents and \textit{uncontrolled} agents such as static and dynamic obstacles. This can lead to either conservative behaviors if all the neighbors are treated as non-cooperative obstacles, or collisions if the obstacles are treated as cooperative, controlled agents. Furthermore, the method in \cite{qin2021learning} 
does not account for changing graph topology in their approximation, which can lead to an incorrect evaluation of the CBF constraints and consequently, failure. 

\textbf{Multi-agent RL} ~ The review paper \cite{dinneweth2022multi} provides a good overview of the recent developments in multi-agent RL (MARL) with applications in safe control design (see \cite{ zhang2019mamps,qie2019joint,everett2018motion}).  
There is also a lot of work on MARL-based approaches with focuses on motion planning \cite{yu2022surprising,cai2021safe,
xiao2022motion,
dai2023socially,pan2022mate,wang2022darl1n}. However, these approaches do not provide safety guarantees due to the reward structure. One major challenge with MARL is designing a reward function for MAS that balances safety and performance. As argued in \cite{wang2022distributed}, MARL-based methods are still in the initial phase of development when it comes to safe multi-agent motion planning. 

\textbf{GNN-based methods} ~ Utilizing the permutation-invariance property, GNN-based methods have been employed for problems involving MAS \cite{yu2023learning,
blumenkamp2022framework,jia2022multi}. The Control Admissiblity Models (CAM)-based framework in \cite{yu2023learning} uses a GNN framework for safe control design for MAS. However, it involves sampling control actions from a set defined by CAM and there are no guarantees that such a set is non-empty, leading to feasibility-related issues of the approach. 
Works such as \cite{tolstaya2021multi,li2020graph} use GNNs for generalization to unseen environments and are shown to work on teams of up to a hundred agents. However, in the absence of an attention mechanism, the computational cost grows with the number of agents in the neighborhood and hence, these methods are not scalable to very large-scale problems (e.g., a team of 1000 agents) due to the computational bottleneck.

The rest of this article is organized as follows. We formulate the MAS control problem in Section \ref{sec:problem}. Then, we present GCBF as a safety certificate for MAS in \Cref{sec:gcbf-theory}, and the framework for learning GCBF and a distributed control policy in \Cref{sec:gcbf+}.  Section \ref{sec: exp imp details} presents the implementation details on the proposed method, while Sections \ref{sec:experiments} and \ref{sec:hardware exps} present numerical and hardware experimental results, respectively. Section \ref{sec:conclusions} presents the conclusions and limitations of the paper, and proposes directions for future work. 

\section{Problem formulation}\label{sec:problem}
\textbf{Notations}~ In the rest of the paper, $\mathbb R$ denotes the set of real numbers and $\mathbb R_+$ denotes the set of non-negative real numbers. We use
$\|\cdot\|$ to denote the Euclidean norm. 
A continuous function $\alpha:\mathbb R\rightarrow\mathbb R$ is an extended class-$\mathcal K$ function if it is strictly increasing and $\alpha(0) = 0$. 
We use $[\cdot]_+$ to denote the function $\max(0, \cdot)$. We drop the arguments $t, x$ whenever clear from the context.
Unless otherwise specified, given a set of vectors $\{x_i\}$ with $x_i\in \mathcal X$ for each $i\in 1, 2, \dots N$ and an index set $\mathcal I$, 
we define $\bar{x}_{\mathcal I} \in \mathcal{X}^{\abs{\mathcal I}}$ as the concatenated vector of the vectors $x_i$ with index $i\in \mathcal I$ from the index set.

{
We consider designing a distributed control framework to drive $N$ agents, each denoted with an index from the set $V_a \coloneqq \{1, 2, \dots, N\}$, to their goal locations in an environment with obstacles while avoiding collisions. 
The motion of each agent is governed by general nonlinear control affine dynamics
\begin{equation} \label{eq:agent_dyn}
    \dot x_i = f_i(x_i) + g_i(x_i)u_i,
\end{equation}
where $x_i\in \mathcal X_i\subset \mathbb R^n$ and $u_i\in \mathcal U_i\subset \mathbb R^m$ are the state, control input for the $i$-th agent, respectively and $f_i:\mathbb{R}^{n}\rightarrow\mathbb{R}^n, g_i:\mathbb R^{n}\rightarrow\mathbb R^{n\times m}$ are assumed to be locally Lipschitz continuous.
{ 
For simplicity, we restrict our discussion to the case when all agents have the same underlying dynamics}, i.e., where $\mathcal X_i = \mathcal X$, $\mathcal U_i = \mathcal U$ and $f_i=f$, $g_i = g$ for all $i\in V_a$. 
Note that it is also possible to apply our approach to heterogeneous MAS.
For convenience, we also define the motion of the entire MAS via the concatenated state vector $\bar{x} \coloneqq [x_1; x_2; \dots; x_N] \in {\mathcal X}^N$ and $\bar{u} \coloneqq [u_1; u_2; \dots; u_N] \in \mathcal{U}^N$,
such that \eqref{eq:agent_dyn} can equivalently be expressed as
\begin{equation} \label{eq:set_dyn}
    \dot{\bar{x}} = \bar{f}(\bar{x}) + \bar{g}(\bar{x}) \bar{u},
\end{equation}
with $\bar{f}$ and $\bar{g}$ defined accordingly.
}

Let $\mathbb P \subset \mathbb R^{\mathfrak n}$ denote the set of positions in an $\mathfrak n$-dimensional environment (i.e., $\mathfrak n=2$ or $\mathfrak n = 3$). We assume that each state $x \in \mathcal{X}$ is associated with a position $p \in \mathbb{P}$, and denote by $p_i\in \mathbb P$ the first $\mathfrak{n}$ elements of $x_i$ corresponding to the positions of each agent $i$.
For each agent $i \in V_a$, we consider a goal position $p^g_i \in \mathbb{P}$, and define $\bar{p}^g$ as the concatenated goal vector.
The observation data consists of $\nrays > 0$ evenly-spaced LiDAR rays originating from each robot and measures the relative location of obstacles within a sensing radius $R>0$.
{ We assume that $R$ is large enough such that there exists a feasible control input that can keep the agents safe once an obstacle is observed.}
For mathematical convenience, we denote the $j$th ray from agent $i$ by $y^{(i)}_j\in \mathcal{X}$, where the first $\mathfrak{n}$ elements of $y^{(i)}_j$ constitute the position of the hitting target $p_j^{(i)}\in\mathbb P$ and the last $n-\mathfrak{n}$ elements are zero padding.
We then denote the aggregated rays as $\bar{y}_i \coloneqq [y_1^{(i)}; \dots; y_{\nrays}^{(i)}] \in \mathcal{X}^{\nrays}$. The inter-agent collision avoidance requirement imposes that each pair of agents maintain a safety distance of $2 r$ while the obstacle avoidance requirement dictates that $|y^{(i)}_j|>r$ for all $j= 1, 2, \dots, \nrays$, where $r > 0$ is the radius of a circle that can contain the entire physical body of each agent. 
The control objective for each agent $i$ is to navigate the obstacle-filled environment to reach its goal $p^g_i$, as described below.  

\begin{Problem}\label{problem: MAS safety}
Design a distributed control policy $\pi_i$ such that, for a set of $N$ agents $\bar{x}$ and non-colliding goal locations $\bar{p}^g$, the following holds for the closed-loop trajectories for each agent $i \in V_a$:
\begin{itemize}
\item \textbf{Safety (Obstacles)}: $\|y^{(i)}_j(t)\| > r,\;  \forall j = 1, \dots, \nrays{}, t\geq 0$, i.e., the agents do not collide with the obstacles.
    \item \textbf{Safety (Other Agents)}: $\|p_i(t) - p_j(t)\| > 2r$ for all $t\geq 0$, $j \neq i$, i.e., the agents do not collide with each other.
\item \textbf{Liveness}: {$\inf\limits_{t \geq 0}\|p_i(t)-p^g_i\| = 0$, i.e., each agent eventually reaches its goal location $p^g_i$}. 
\end{itemize}
\end{Problem}

To solve \Cref{problem: MAS safety}, we consider the existence of a nominal controller that satisfies the liveness property but not necessarily the safety property, and construct a GCBF-based distributed control policy to additionally satisfy the safety property.

\section{GCBF: a safety certificate for MAS}\label{sec:gcbf-theory}

Based on the algorithm in \cite{zhang2023distributed} (GCBFv0), we propose an improved algorithm, termed GCBF+, to train a \emph{graph} CBF (GCBF) that encodes the collision-avoidance constraints based on the graph structure of MAS. We use GNNs to learn a candidate GCBF jointly with the collision-avoidance control policy. Our GNN architecture is capable of handling a variable number of neighbors and hence results in a distributed and scalable solution to the safe MAS control problem.

\subsection{Safety for arbitrary sized MAS via graphs} \label{subsec:MAS_graph}

We first review the notion of CBF commonly used in literature for safety requirements \cite{ames2019control}. 
Consider a system $\dot x = F(x, u)$ where $x\in \mathcal X \subset \mathbb R^n$, $u\in \mathcal U\subset \mathbb R^m$ and $F:\mathbb R^n\times \mathbb R^m\to \mathbb R^n$. Let $\mathcal C\subset\mathcal X$ be the $0$-superlevel set of a continuously differentiable function $h: \mathcal X\rightarrow\mathbb R$, i.e., $\mathcal C= \{x\in \mathcal X: h(x) \geq 0\}$. Then, $h$ is a CBF if there exists an extended class-$\mathcal{K}$ function $\alpha: \mathbb R\rightarrow\mathbb R$ such that:
\begin{equation}\label{eq: cbf-descent-cond}
    \sup_{u\in\mathcal U} \left[\frac{\partial h}{\partial x} F(x,u) + \alpha\left(h(x)\right)\right]\geq 0, \quad \forall x\in\mathcal X. 
\end{equation}
Let $\mathcal S\subset \mathcal X$ denote a safe set with the objective that the system trajectories do not leave this set. If $\mathcal C\subset \mathcal S$, then the existence of a CBF implies
the existence of a control input $u$ that keeps the system safe \cite{ames2017control}. 

Based on the notion of CBF, we define the new notion of a GCBF to encode safety for MAS of any size. To do so, we first define the graph structure we will use in this work.

A directed graph is an ordered pair $G = (V, E)$, where $V$ is the set of nodes, and $E \subset \{ (i, j) \mid i \in V_a,\, j \in V \}$ is the set of edges representing the flow of information from a \textit{sender} node (henceforth called a neighbor) $j$ to a \textit{receiver} agent $i$. 
Let $\tilde{\mathcal{N}}_i$ denote the set of neighbors for agent $i\in V_a$.
For the considered MAS, we define the set of nodes $V = V_a \cup V_o$ to consist of the agents $V_a$ and the hitting points of all the LiDAR rays from all agents denoted as $V_o$.
The edges are defined between each observed point and the observing agent when the distance between them is within a sensing radius $R > 2r > 0$. 

Given $\mathfrak{n}$, sensing radius $R$, safety radius $r$ and $\nrays$, define $M - 1 \in \mathbb{N}$ as the maximum number of sender neighbors that each receiver agent node can have while all the agents in the neighborhood remain safe.
For simplicity, define $\mathcal{N}_i\subseteq\tilde{\mathcal N}_i$ as the set of $M$ closest neighboring nodes to agent $i$ which also includes agent $i$.\footnote{For breaking ties, the agent with the smaller index is chosen.} Next, define $\bar{x}_{\mathcal{N}_i} \in \mathcal{X}^{M}$ as the concatenated vector of $x_i$ and the neighbor node states with fixed size $M$ that is padded with a constant vector if $\big\lvert \mathcal{N}_i \big\rvert < M$.

\begin{Remark} \label{remark:R_small}
We define $M$ as above so that, for any $i\in V_a$,
changes in the neighboring indices $\mathcal{N}_i$ can only occur without collision at a distance $R$ (see Appendix \ref{app: R_small}).

\end{Remark}

\subsection{Graph Control Barrier Functions}

We define the safe set $\mathcal S_N \subset \mathcal{X}^N$ of an $N$-agent MAS as the set of MAS states $\bar{x}$ that satisfy the safety properties in Problem \ref{problem: MAS safety}, i.e.,
\begin{equation} \label{eq:SN_def}
\begin{aligned}
    \mathcal S_N \coloneqq \Big\{ \bar{x} \in \mathcal{X}^N \;\Big|
    &\; \Big( \norm{y_j^{(i)}} > r,\; \forall i \in V_a, \forall j\in n_\mathrm{rays} \Big) \bigwedge \\
    &\Big( \min_{i, j \in V_a, i \neq j} \norm{p_i - p_j} > 2r \Big) \Big\}.
\end{aligned}
\end{equation}
Then, the unsafe, or avoid, set of the MAS $\mathcal{A}_{N} = \mathcal{X}^N \setminus \mathcal{S}_N$ is defined as the complement of $\mathcal S_N$. 

{
We now introduce the notion of GCBF for encoding safety for MAS. 
We impose that for a given agent $i\in V_a$, a node $j$ where $\norm{ p_i - p_j } \geq R$ does not affect the GCBF $h$ so that the resulting $h$ is smooth. Specifically, for any neighborhood set $\mathcal{N}_i$, let $\mathcal{N}^{<R}_i$ denote the set of neighbors in $\mathcal{N}_i$ that are strictly inside the sensing radius $R$ as
\begin{equation}
    \mathcal{N}^{<R}_i \coloneqq \{ j : \norm{p_i - p_j} < R,\; j \in \mathcal{N}_i \}.
\end{equation}
Now, we are ready to define GCBF formally.
}

{\begin{Definition}[\textbf{GCBF}]\label{def: gcbf}
A continuously differentiable function $h : \mathcal{X}^M \to \mathbb{R}$ is termed as a Graph CBF (GCBF) if there exists an extended class-$\mathcal K$ function $\alpha$ and a control policy $\pi_i: \mathcal{X}^M \to \mathcal{U}$ for each agent $i \in V_a$ of the MAS such that, for all $\bar{x} \in \mathcal{X}^N$ with $N \geq M$,
\begin{equation}\label{eq:graph CBF}
        \dot h(\bar{x}_{\mathcal{N}_i}) +\alpha( h( \bar{x}_{\mathcal{N}_i} ) )\geq 0 \quad \forall i \in V_a,
\end{equation}
where
\begin{equation} \label{eq:hdot_def}
    \dot h(\bar x_{\mathcal N_i}) = \sum_{j\in 
\mathcal{N}_i}\frac{\partial h(\bar x_{\mathcal N_i})}{\partial x_j}\left(f(x_j) + g(x_j) u_j \right),
\end{equation}
with $u_j = \pi_j(\bar x_{\mathcal N_j})$, and the following two conditions hold:
\begin{enumerate}
    \item The gradient of $h$ with respect to nodes $R$ away is $0$, i.e.,
    \begin{equation} \label{eq:ass_deriv_zero}
        \frac{\partial h}{\partial x_j}( \bar{x}_{\mathcal{N}_i} ) = 0 \quad \forall j \in \mathcal{N}_i \setminus \mathcal{N}^{<R}_i.
    \end{equation}
\item The value of $h$ does not change when restricting to neighbors that are in $\mathcal{N}^{<R}_i$, i.e.,
\begin{equation} \label{eq:ass_same_h}
        h( \bar{x}_{\mathcal{N}_i} ) = h( \bar{x}_{\mathcal{N}^{<R}_i} ).
    \end{equation}
\end{enumerate}
\end{Definition}
}

{

\begin{Lemma} \label{lemma:cts}
Given a GCBF $h$, the function $t \mapsto h( \bar{x}_{\mathcal{N}_i(t)}(t))$ is a continuously differentiable function despite $\bar{x}_{\mathcal{N}_i(t)}(t)$ having discontinuities whenever the set of neighboring indices $\mathcal{N}_i(t)$ for agent $i$ changes.
\end{Lemma}
The proof of \Cref{lemma:cts} is provided in Appendix \ref{app: h_cts_proof}.
}
\begin{Remark}\label{rmk:attn}
    One way of satisfying the conditions 1) and 2) in \Cref{def: gcbf} is by taking $h$ to be of the form
    \begin{equation} \label{eq:attn_thing}
        h(\bar{x}_{\mathcal{N}_i}) = \xi_1\bigg( \sum_{j \in \mathcal{N}_i} w(x_i, x_j)\, \xi_2(x_i, x_j) \bigg),
    \end{equation}
    where $\xi_1: \mathbb R^{\rho} \to \mathbb R$ and $\xi_2: \mathcal X \times \mathcal X \to \mathbb R^\rho$ are two encoding functions with $\rho$ the dimension of the feature space, and $w : \mathcal{X} \times \mathcal{X} \to \mathbb{R}$ is a continuously differentiable function such that $w(x_i, x_j) = 0$ and $\frac{\partial w}{\partial x_j}(x_i, x_j) = 0$ whenever $\norm{p_i - p_j} \geq R$.
    In practice, we use graph attention \cite{li2019graph}, which takes the form \eqref{eq:attn_thing}, to realize conditions 1) and 2) in \Cref{def: gcbf}, which we introduce later in \Cref{sec:learn-gcbf}. 
\end{Remark}

With \Cref{def: gcbf}, a GCBF certifies the forward invariance of its $0$-superlevel set under a suitable choice of control inputs. For a GCBF $h$, let ${\mathcal{B}_h} \subset \mathcal{X}^M$ denote the $0$-superlevel set of $h$
\begin{equation}
    {\mathcal{B}_h} \coloneqq \{ \tilde{x} \in \mathcal{X}^M \mid h( \tilde{x} ) \geq 0 \},
\end{equation}
and define $\mathcal{C}_N \subset \mathcal{X}^N$ as the set of $N$-agent MAS states where $\bar{x}_{\mathcal{N}_i}$ lie inside ${\mathcal{B}_h}$ for all $i \in V_a$, i.e.,
\begin{align}\label{eq:C_N}
    \mathcal{C}_N &\coloneqq \bigcap_{i=1}^N \mathcal{C}_{N, i},
\end{align}
where
\begin{align}\label{eq:C_Ni} 
\mathcal{C}_{N,i} &\coloneqq \{ \bar{x} \in \mathcal{X}^N \mid \bar{x}_{\mathcal{N}_i} \in {\mathcal{B}_h} \}. 
\end{align}

\begin{figure}
    \centering
    \includegraphics[width=0.75\columnwidth]{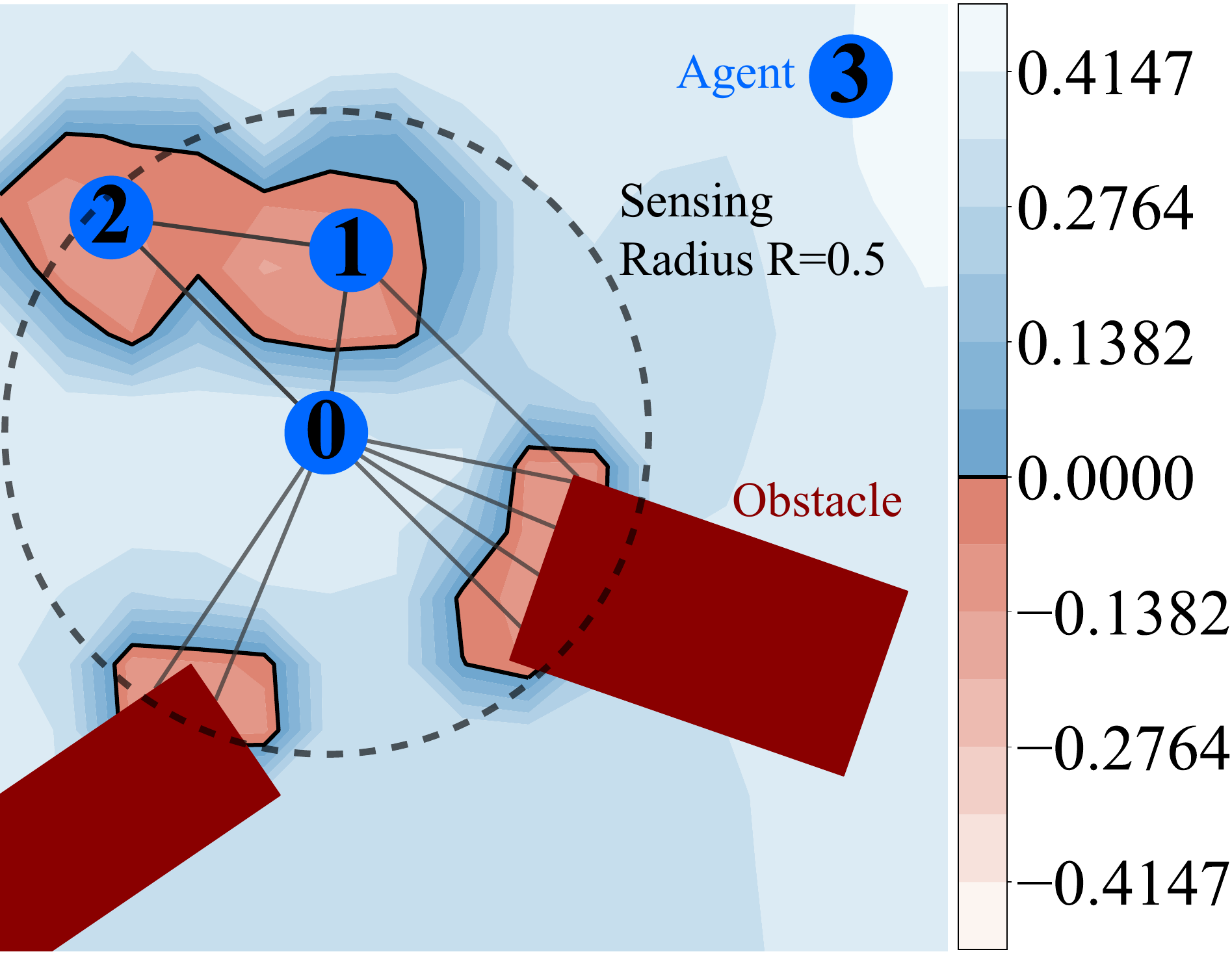}
    \caption{\textbf{GCBF contours: }
For a fixed neighborhood $\mathcal{N}_0$ of agent $0$, we plot the contours of the learned GCBF $h_\theta$, {projected on $xy$-plane}, by varying the position $p_0$ of agent $0$.
Since agent $3$ is outside of agent $0$'s sensing radius, i.e., not a neighbor of agent $0$, it does not contribute to the value of $h(\bar{x}_{\mathcal{N}_0})$.
}
    \label{fig:cbf}
\end{figure}

We now state the result of the safety guarantees of GCBF. 
\begin{Theorem} \label{thm: safety result}
    Suppose $h$ is a GCBF following \Cref{def: gcbf} and {$\mathcal{C}_N \subset \mathcal{S}_N$} for each $N \in \mathbb{N}$. Then, for any $N \in \mathbb{N}$, the resulting closed-loop trajectories of the MAS with initial conditions $\bar{x}(0) \in \mathcal{C}_N$ under any locally Lipschitz continuous control input $\bar{u} : \mathcal{X}^N \to \mathcal{U}^N_{\mathrm{safe}}$  satisfy $\bar{x}(t) \in \mathcal{S}_N$ for all $t \geq 0$, where
    \begin{equation}
        \mathcal{U}^N_{\mathrm{safe}} \coloneqq \left\{ \bar{u} \in \mathcal{U}^N\; \Big|\; \dot h( \bar{x}_{\mathcal{N}_i} ) + \alpha(h(\bar{x}_{\mathcal{N}_i})) \geq 0,\, \forall i \in V_a \right\},
    \end{equation}
    with the time derivative of $h$ given as in \eqref{eq:hdot_def}.
\end{Theorem}
As a result of \Cref{thm: safety result}, the set $\mathcal{C}_N$, for any $N \in \mathbb{N}$, is a safe control invariant set \cite{blanchini1999set}. An example of a GCBF is shown in \Cref{fig:cbf}.

Unlike traditional methods of proving forward invariance using CBFs \cite{ames2017control}, the proof of \Cref{thm: safety result} is more involved as it must handle the dynamics discontinuities that occur when the neighborhood $\mathcal{N}_i$ of agent $i\in V_a$ changes. We make use of \Cref{lemma:cts} to handle such discrete jumps. The proof of \Cref{thm: safety result} is provided in Appendix \ref{app: safety proof}.

\begin{Remark}
Note that \Cref{thm: safety result} proves that a GCBF can certify the safety of a MAS of \textit{any} size $N \in \mathbb{N}$. 
This is in contrast to the result in \cite{zhang2023distributed}, which only proves that a GCBF can {certify} safety for a specific $N$. {As a result, the GCBF from \cite{zhang2023distributed}, or other notions of CBFs in the prior work, trained on, say, $8$ agents, cannot theoretically certify safety when used for a MAS of $1000$ agents. However, the proposed framework allows such theoretical certification to carry over from a smaller-sized MAS to a larger-sized MAS.}
This brings the theoretical understanding of the proposed algorithm closer to the empirical results, where we observe the new GCBF+ algorithm can scale to over $1000$ agents despite being trained with only $8$ agents.
\end{Remark}

\begin{Remark} \label{remark:CNi_subset}
Note that the individual $\mathcal{C}_{N,i}$ do not all need to be a subset of $\mathcal{S}_N$ as long the intersection $\mathcal{C}_N \subset \mathcal{S}_N$ in \Cref{thm: safety result}.
    For example, if $\mathcal{S}_N$ can be written as the intersection of sets $\mathcal{S}_{N,i}$, i.e., $\mathcal S_N = \bigcap_{i=1}^N \mathcal{S}_{N,i}$, then it is sufficient that $\mathcal{C}_{N,i} \subseteq \mathcal{S}_{N,i}$ for all $i \in V_a$ to obtain that $\mathcal{C}_{N} \subseteq \mathcal{S}_N$, since
    \begin{equation}
        \mathcal{C}_{N} \coloneqq \bigcap_{i=1}^N \mathcal{C}_{N,i} \subseteq \bigcap_{i=1}^N \mathcal{S}_{N,i} \eqqcolon \mathcal{S}_N.
    \end{equation}
    In practice, we take this approach and define
$\mathcal{S}_{N,i}$ for each agent $i\in V_a$ as
\begin{equation}
    \begin{aligned}
        \mathcal{S}_{N,i} \coloneqq \Big\{ \bar{x} \in \mathcal{X}^N \;\Big|
        &\; \Big( \norm{y_j^{(i)}} > r,\; \forall j\in n_\mathrm{rays} \Big) \bigwedge \\
        &\Big( \min_{j \in V_a, i \neq j} \norm{p_i - p_j} > 2r \Big) \Big\}.
    \end{aligned}
    \end{equation}
\end{Remark}

\subsection{Safe control policy synthesis}
For the multi-objective Problem \ref{problem: MAS safety}, in the prior work GCBFv0 \cite{zhang2023distributed}, we used a hierarchical approach for the goal-reaching and the safety objectives, where a nominal controller was used for the liveness requirement. 
During training, a term is added to the loss function so that the learned controller is as close to the nominal controller as possible. In execution, GCBFv0 uses a switching mechanism to switch between the nominal controller for goal reaching and the learned controller for collision avoidance. However, the added loss term corresponding to the nominal controller competes with the CBF loss terms for safety, often sacrificing either safety or goal-reaching.

In this work, we use a different mechanism for encoding the liveness property. 
Given a nominal controller $u_i^{\mathrm{nom}} = \pinom(x_i, p^g_i)$ for the goal-reaching objective, we design a controller that satisfies the safety constraint using an optimization framework that minimally deviates from a nominal controller that only satisfies the liveness requirements.\footnote{In this work, we use simple controllers, such as linear quadratic regulator (LQR) and PID-based nominal controllers in our experiments.} Given a GCBF $h$ and an extended class-$\mathcal K$ function $\alpha$, a solution to the following centralized optimization problem
\begin{subequations}\label{eq:opt control policy}
\begin{align}
    \min_{\bar{u}} \quad & \sum_{i \in V_a} \|u_i- u_i^{\mathrm{nom}}\|^2,\label{eq: obj-cbf-qp} \\
    \text{s.t.} \quad &  u_i \in \mathcal{U}, \quad \forall i \in V_a, \\
& \!\!\sum_{j \in \mathcal{N}_i} \frac{\partial h}{\partial x_j} \big(f(x_j) + g(x_j)u_j\big) \geq -\alpha( h(\bar{x}_{\mathcal{N}_i}) ), \; \forall i \in V_a, \label{eq:CBF constraint pi nom}
\end{align}
\end{subequations}
keeps all agents within the safety region \cite{ames2017control}. Note that 
\eqref{eq:CBF constraint pi nom} is linear in the decision variables $\bar{u}$. When the input constraint set $\mathcal U$ is a convex polytope, \eqref{eq:opt control policy} is a quadratic program (QP) and can be solved efficiently online for robotics applications \cite{ames2019control}.
We define the policy $\piqp: \mathcal{X}^N \to \mathcal{U}^N$ as the solution of the QP \eqref{eq:opt control policy} at the MAS state $\bar{x}$.
Note that \eqref{eq:opt control policy} is not a {distributed} framework for computing the control policy, since $u_i$ is indirectly coupled to the controls of {all} other agents via the constraint \eqref{eq:CBF constraint pi nom}.
Although there is work on using distributed QP solvers to solve \eqref{eq:opt control policy} (see e.g., \cite{pereiradecentralized}), these approaches are not easy to use in practice for real-time control synthesis of large-scale MAS.
To this end, we use an NN-based control policy that does not require solving the centralized QP online. We present the training setup for jointly learning both GCBF and a distributed safe control policy in the next section.

\begin{figure*}
    \centering
    \includegraphics[width=\linewidth]{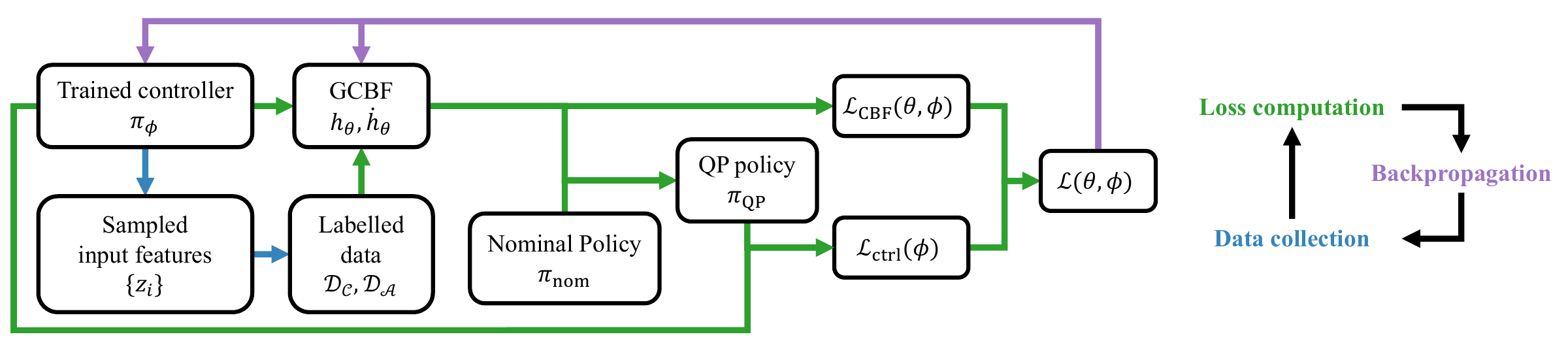}
    \caption{\textbf{GCBF+ training architecture: }The sampled input features are labeled as safe control invariant $\mathcal D_{\mathcal C}$ and unsafe $\mathcal D_{\mathcal A}$ using the previous step learned control policy $\pi_\phi$. A nominal control policy $\pi_{\mathrm{nom}}$ for goal reaching is used in a CBF-QP with the previously learned GCBF $h_\theta$ to generate $\pi_{\mathrm{QP}}$. Finally, the QP policy and the GCBF conditions are used to define the loss $\mathcal L$.}
\label{fig:algo-structure}
\end{figure*}

\section{GCBF+: framework for learning GCBF and distributed control policy}\label{sec:gcbf+}
\subsection{Neural GCBF and distributed control policy}\label{sec:learn-gcbf}

Drawing on the graph representation of arbitrary sized MAS introduced in \Cref{subsec:MAS_graph}, we apply GNNs to learn a GCBF $h_\theta$ and distributed control policy $\pi_\phi$ for parameters $\theta, \phi$. We transform the MAS graph into \textit{input features} to be used as the GNN input by constructing \textit{node features} and \textit{edge features} corresponding to the nodes and edges of the graph $G$. To learn a goal-conditioned control policy that can reach different goal positions, we introduce a goal node and an edge between each agent and their goal in the input features.

\textbf{Node features and edge features }
The node features $v_i\in \mathbb R^{\rho_v}$ encode information specific to each node. In this work, we take $\rho_v = 3$ and use the node features $v_i$ to one-hot encode the type of the node as either an agent node, goal node or LiDAR ray hitting point node.
The edge features $e_{ij}\in \mathbb R^{\rho_e}$, where $\rho_e > 0$ is the edge dimension, are defined as the information shared from node $j$ to agent $i$, which depends on the states of the nodes $i$ and $j$. 
Since the safety objective depends on the relative positions, one component of the edge features is the relative position $p_{ij} = p_j - p_i$. The rest of the edge features can be chosen depending on the underlying system dynamics, e.g., relative velocities for double integrator dynamics. 

\begin{figure}
    \centering
\includegraphics[width=0.7\columnwidth]{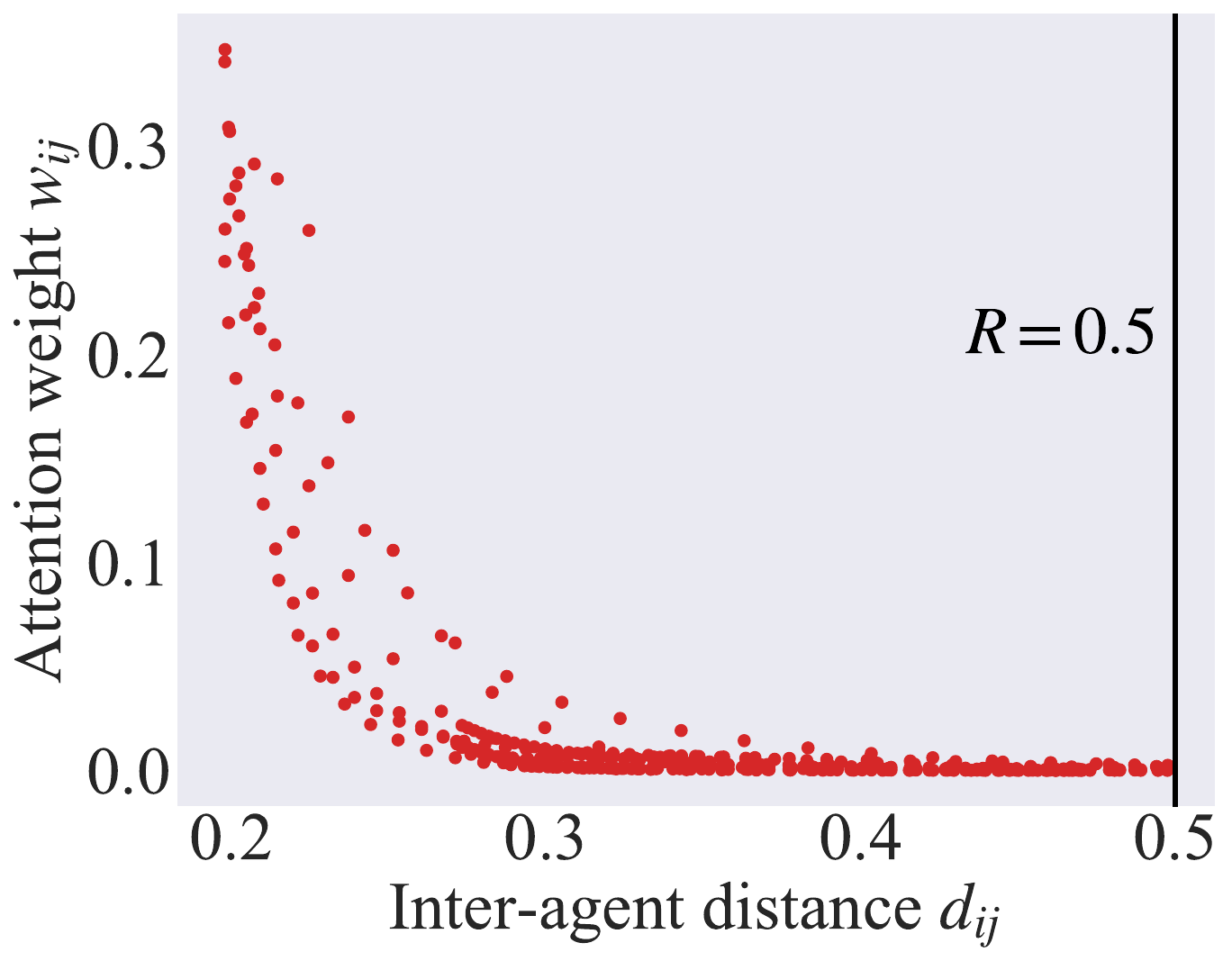}
\caption{\textbf{Satisfaction of \Cref{def: gcbf} in practice: }
    The attention weights $w_{ij}$ in \eqref{eq:attention_aggregate} are plotted against inter-agent distances $d_{ij}$ for sensing radius $R=0.5$ from multiple trajectories. The weight $w_{ij}$ approaches $0$ as the inter-agent distance $d_{ij}$ approaches $R$ without explicit supervision, showing that GCBF+ automatically learns to satisfy conditions 1) and 2) in \Cref{def: gcbf} via the approach outlined in \Cref{rmk:attn}.
}
\label{fig:attn}
\end{figure}

\textbf{GNN structure }
Thanks to the ability of GNN to take variable-sized inputs, we do not need to add padding nor truncate the input of the GCBF into a fixed-sized vector when $\big\lvert \tilde{\mathcal{N}}_i \big\rvert \not= M$. We define the input of $h_\theta$ to be input features $z_i=\left[z_{i1}^\top,\dots,z_{i\lvert\tilde{\mathcal{N}}_i\rvert}^\top\right]^\top$, where $z_{ij}=\left[v_i,v_j,e_{ij}\right]^\top$.
In GNN, we first encode each $z_{ij}$ to the feature space via an MLP $\psi_{\theta_1}$ to obtain $q_{ij} = \psi_{\theta_1}(z_{ij})$. Next, we use graph attention~\cite{li2019graph} to aggregate the features of the neighbors of each node, i.e.,
\begin{equation} \label{eq:attention_aggregate}
    q_i=\sum_{j\in\tilde{\mathcal{N}}_i} \underbrace{\softmax\big( \psi_{\theta_2}(q_{ij}) \big)}_{w_{ij}}\, \psi_{\theta_3}(q_{ij}),
\end{equation}
where $\psi_{\theta_2}$ and $\psi_{\theta_3}$ are two NNs parameterized by $\theta_2$ and $\theta_3$. $\psi_{\theta_2}$ is often called ``gate'' NN in literature \cite{li2015gated}, and the resulting \textit{attention weights} $w_{ij} \in [0, 1]$ (with $\sum\limits_{j \in \mathcal{N}_i} w_{ij} = 1$) encode how important the sender node $j$ is to agent $i$. Note that applying attention in GNNs is crucial for satisfying conditions 1) and 2) in \Cref{def: gcbf}, which are satisfied in our experiments since the attention weights and their derivatives go to $0$ as the inter-agent distance goes to $R$
without any additional supervision (see \Cref{rmk:attn} and \Cref{fig:attn}).
The aggregated information $q_i$ in \eqref{eq:attention_aggregate} is then passed through another MLP $\psi_{\theta_4}$ to obtain the output value $h_i = \psi_{\theta_4}(q_i)$ of the GCBF for each agent.
We use the same GNN structure for the control policy $\pi_\phi$.
Since the input features $z_i$ for agent $i$ only depend on the neighbors $\tilde{\mathcal{N}}_i$, the $\pi_\phi$ is distributed unlike \eqref{eq:opt control policy}.

\subsection{GCBF+ loss functions}
We train the GCBF $h_\theta$ and the distributed controller $\pi_\phi$ by minimizing the sum of the CBF loss $\mathcal{L}_{\mathrm{CBF}}(\theta, \phi)$ and the control loss $\mathcal{L}_{\mathrm{ctrl}}(\phi)$:
\begin{equation}
    \mathcal{L}(\theta, \phi) = \mathcal{L}_{\mathrm{CBF}}(\theta, \phi) + \mathcal{L}_{\mathrm{ctrl}}(\phi).
\end{equation} 
The CBF loss $\mathcal{L}_{\mathrm{CBF}}(\theta, \phi)$ and the control loss $\mathcal{L}_{\mathrm{ctrl}}(\phi)$ are defined as sums over each agent as
\begin{subequations}
\begin{equation}
    \mathcal{L}_{\mathrm{CBF}}(\theta, \phi) \coloneqq \sum_{i \in V_a} \mathcal{L}_{\mathrm{CBF}, i}(\theta, \phi), \label{eq: cbf loss}
\end{equation}
\begin{equation}
    \mathcal{L}_{\mathrm{ctrl}}(\phi) \coloneqq \sum_{i \in V_a} \mathcal{L}_{\mathrm{ctrl}, i}(\phi).
\end{equation}
\end{subequations}
Denote by $\mathcal D_{\mathcal C,i}, \mathcal D_{\mathcal A,i}$ the set consisting of labeled input features in the safe control invariant region and unsafe region, respectively.
The CBF loss $\mathcal{L}_{\mathrm{CBF}, i}$ penalizes violations of the GCBF condition \eqref{eq:graph CBF} and the (sufficient) safety requirement that the $0$-superlevel set $\mathcal C_{N, i}$ is a subset of $S_{N, i}$ (see \Cref{remark:CNi_subset}):
\begin{multline} \label{eq:L_cbf_i}
    \mathcal{L}_{\mathrm{CBF}, i}(\theta, \phi) \coloneqq \eta_\mathrm{deriv}\sum_{z_i}\left[\gamma-\dot h_\theta(z_i)-\alpha(h_\theta(z_i))\right]^+\\
        +\sum_{z_i\in\mathcal D_{\mathcal C, i}}\left[\gamma-h_\theta(z_i)\right]^+ + \sum_{z_i\in\mathcal D_{\mathcal A, i}}\left[\gamma+h_\theta(z_i)\right]^+,
\end{multline}
where $\gamma > 0$ is a hyper-parameter to encourage strict inequalities and $\eta_\mathrm{deriv} > 0$ weighs the GCBF condition \eqref{eq:graph CBF}.
We use a class-$\mathcal K$ function $\alpha(h) = \bar{\alpha} h$ for a constant $\bar{\alpha} > 0$. From hereafter, we abuse the notation and use $\alpha$ to refer to $\bar{\alpha}$.
The control loss $\mathcal{L}_{\mathrm{ctrl}, i}$ encourages the learned controller $\pi_\phi$ to remain close to the QP controller $\piqp$ (the solution to the QP \eqref{eq:opt control policy} with $h$ being the learned $h_\theta$ in the previous learning step), which is the closest control to $\pinom$ that maintains safety:
\begin{equation}
    \mathcal{L}_{\mathrm{ctrl}, i}(\phi) \coloneqq \eta_\mathrm{ctrl} \left\|\pi_\phi(z_i)-\pi_{\mathrm{QP}, i}(\bar x)\right\|,
\end{equation}
where $\eta_\mathrm{ctrl}$ is the control loss weight and $\pi_{\mathrm{QP}, i}(\bar x) \in \mathcal U$ is the $i$-th component of $\piqp(\bar x)$. 
{ In practice, it is possible that the QP is infeasible during training.
To this end, a relaxation term $\zeta\geq 0$ is added to the left-hand side of constraint \eqref{eq:CBF constraint pi nom} along with a penalty term for $\zeta$ with a large coefficient added to the objective function \eqref{eq: obj-cbf-qp}.
Once the CBF loss \eqref{eq: cbf loss} converges to zero, the QP is feasible on the sampled data points.
}

One of the challenges of evaluating the loss function $\mathcal L$ is computing $\dot h_\theta$. Similar to \cite{yu2023learning}, we estimate $\dot h_\theta$ by $\left(h_\theta(z_i(t_{k+1})) - h_\theta(z_i(t_k))\right)/\delta t$, where $\delta t = t_{k+1}-t_k$ is the timestep. Estimating $\dot{h}_\theta$ may be problematic if the set of neighbors $\tilde{\mathcal{N}}_i$ changes between $t_k$ and $t_{k+1}$. 
However, the learned attention weights satisfy conditions 1) and 2) in \Cref{def: gcbf} as noted in \Cref{sec:learn-gcbf}. Consequently, \Cref{lemma:cts} implies that $h_\theta$ is continuously differentiable, and our estimate of $\dot{h}_\theta$ is well behaved.
Note that $\dot h_\theta$ includes the inputs from agent $i$ and the neighbor agents $j\in \mathcal N_i$. Therefore,  when we use gradient descent and backpropagate $\mathcal{L}_i(\theta,\phi)$ during training, the gradients are passed to not only the controller of agent $i$ but also the controllers of all neighbors in $\tilde{\mathcal N}_i$.\footnote{We re-emphasize the fact the neighbors' inputs are not required for $\pi_\phi$ during testing.} More details on the training process are provided in \Cref{sec:implmnt details}. The training architecture is summarized in \Cref{fig:algo-structure}.

{
\begin{Remark} \label{remark:new_ctrl_loss} \textbf{Benefits of the new loss function} ~ In the prior works ~\cite{zhang2023distributed,qin2021learning}, the nominal policy $\pinom$ is used instead of $\piqp$ in $\mathcal L_{\mathrm{ctrl}}$. As a result, these approaches suffer from a trade-off between safety and goal-reaching and often learn a sub-optimal policy that compromises safety for liveness, or liveness for safety. In contrast, our proposed method allows the loss to converge to zero, and thus, does not have this trade-off.
\end{Remark}
}

\Cref{fig:eta ctrl traj} plots the trajectories of an agent in the presence of an obstacle under learned policies with $\pi_\mathrm{nom}$ and $\pi_\mathrm{QP}$ in the $\mathcal L_\mathrm{ctrl}$. When $\pi_\mathrm{nom}$ is used, for small values of $\eta_\mathrm{ctrl}$, the learned controller over-prioritizes safety leading to conservative behavior (the top left plot). For large $\eta_\mathrm{ctrl}$, the learned policy over-prioritizes goal reaching, leading to unsafe behaviors (the top right plot). Using an optimal $\eta_\mathrm{ctrl}$, it is possible to get a desirable behavior as in the top middle plot. In contrast, when $\pi_\mathrm{QP}$ is used in the control loss, the goal-reaching and the safety losses do not compete with each other and it is possible to get a desirable behavior without extensive hyper-parameter tuning as can be seen in the bottom plots in \Cref{fig:eta ctrl traj}. Note that when $\pi_\mathrm{nom}$ is used, even with the optimal value of $\eta_\mathrm{ctrl}$, the learned input (shown with black arrows) does not align with the nominal input (shown with orange arrows), meaning that the total loss may not converge to zero in such formulations unless the nominal policy is also safe. On the other hand, when $\pi_\mathrm{QP}$ is used, the two control inputs have much more similar values, which allows for the total loss to converge to $0$.

\begin{figure}
    \definecolor{pol_color}{HTML}{e6830b}
    \centering  
\includegraphics[width=\linewidth]{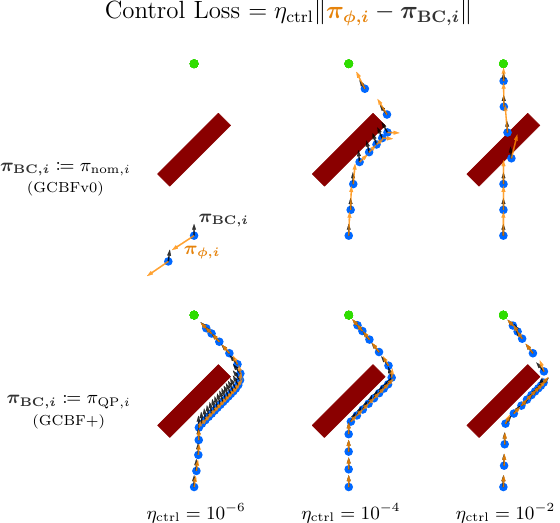}
\caption{\textbf{Comparison of the choice of control loss: }The learned policy $\boldsymbol{\textcolor{pol_color}{\pi_\phi}}$ is sensitive to the choice of $\eta_{\mathrm{ctrl}}$ when using $\pi_{\mathrm{nom}}$ as in GCBFv0 \cite{zhang2023distributed} (top) since a learned policy $\boldsymbol{\textcolor{pol_color}{\pi_\phi}}$ that is close to $\pi_{\mathrm{nom}}$ may be unsafe and not satisfy the GCBF conditions \eqref{eq:graph CBF}. Consequently, choosing $\eta_{\mathrm{ctrl}}$, which controls the relative weight between $\mathcal{L}_{\mathrm{CBF}}$ and $\mathcal{L}_{{ctrl}}$, becomes a balancing act of staying close to $\pi_{\mathrm{nom}}$ while remaining safe.
    In contrast, by definition of $\pi_{\mathrm{QP}}$ \eqref{eq:opt control policy}, the control $\boldsymbol{\textcolor{pol_color}{\pi_\phi}}$ is already safe. Hence, the safety of $\boldsymbol{\textcolor{pol_color}{\pi_\phi}}$ is not sensitive to $\eta_{\mathrm{ctrl}}$ when using $\pi_{\mathrm{QP}}$ for the control loss, as in GCBF+.
}
\label{fig:eta ctrl traj}
\end{figure}

\subsection{Data collection and labeling}\label{sec:training-data}

The training data $\{z_i\}$ is collected over multiple scenarios and the loss is calculated by evaluating the CBF conditions on each sample point. We use an on-policy strategy to collect data by periodically executing the learned controller $\pi_\phi$, which helps align the train and test distributions. 

When labeling the training data as $\mathcal D_{\mathcal C, i}$ or $\mathcal D_{\mathcal A, i}$, it is important to note that an input feature $z_i$ that is not in a collision may be unable to prevent an inevitable collision in the future under actuation limits.
For example, under acceleration limits, an agent that is moving too fast may not have enough time to stop, resulting in an inevitable collision.
Therefore, we cannot na\"ively label all the input features as $\mathcal D_{\mathcal C, i}$ if they are not in any collision at the current time step
unless there exists a control policy that can keep them safe in the future. 
However, as noted in \cite{so2023train,hsu2021safety,fisac2019bridging,schafer2023scalable}, computing an infinite horizon control invariant set for high dimensional nonlinear systems is computationally challenging, and often, various approximations are used for such computations. In this work, we use a finite-time reachable set as an approximation. 
At any given learning step, given a graph $G$, and the corresponding input features $z_i, i\in V_a$, we use the learned policy from the previous iteration to propagate the system trajectories for $T$ timesteps. If the entire trajectory remains in the safe set $\mathcal S_{N, i}$ for agent $i$, then $z_i$ is added to the set $\mathcal D_{\mathcal C, i}$. If there exist collisions in $z_i$, it is added to the set $\mathcal D_{\mathcal A, i}$. Otherwise, it is left unlabeled. 
As $T \to \infty$, this recovers the infinite horizon control invariant set, but is not tractable to compute. Instead, we choose a large but finite value of $T$ as a hyperparameter in our numerical experiments.  

{ The training data can be collected and labeled efficiently as follows. With the learned policy from the previous iteration, we sample feature trajectories $\{z_i^0, \dots, z_i^{\bar T}\}$ for some $\bar T \geq T$. If any feature $z_i^{t}$ is unsafe, it is added to $\mathcal D_{\mathcal A, i}$. If $T$ features $z_i^{t+1}, \dots, z_i^{t+T}$ after a feature $z_i^t$ including itself are all safe, then $z_i^t$ is added to $\mathcal D_{\mathcal C, i}$. Otherwise, it remains unlabeled.}

\begin{Remark} \textbf{Importance of control-invariant labels} ~
Note that GCBFv0 ~\cite{zhang2023distributed} does \textit{not} use the concept of the safe control invariant set during training. Instead, similar to prior works ~\cite{qin2021learning,dawson2022safe}, the learned CBF is enforced to be positive on the entire safe set $\mathcal S_{N, i}$, even for states that are not control invariant. 
{ Prior works attempt to alleviate these issues by estimating the control-invariant using a shrunk safe set and introducing a fixed margin between the safe and unsafe sets. However, for high relative degree dynamics, this margin is state-dependent. Therefore, this estimation of the control-invariant set does not result in a true control-invariant set.}
As noted in \cite{so2023train}, if the safe set $\mathcal S_{N, i}$ is not control-invariant, then no valid CBF exists that is positive on $\mathcal S_{N, i}$. We later investigate the importance of the quality of the control-invariant labels (as controlled by $T$) in \Cref{subsec:sensitivity} and find that poor approximations of the control-invariant set $\mathcal C_{N, i}$ via small values of $T$ leads to large drops in the safety rate. This provides some insight into the performance improvements of GCBF+ over GCBFv0.
\end{Remark}

\section{Experiments: implementation details}\label{sec: exp imp details}

In this section, we introduce the details of the experiments, including the implementation details of GCBF+ and the baseline algorithms, and the details of each environment. 

\textbf{Environments}~ We conduct experiments on five different environments consisting of three 2D environments (SingleIntegrator, DoubleIntegrator, DubinsCar) and two 3D environments (LinearDrone, CrazyflieDrone). See Appendix \ref{app: env details} for details. The parameters are $R = 0.5, r = 0.05$ in all environments. 
The total time steps for each experiment is $4096$.

\textbf{Evaluation criteria}~ We use safety rate, reaching rate, and success rate as the evaluation criteria for the performance of a chosen algorithm. { The \textbf{safety rate} is defined as the ratio of agents not colliding with either obstacles or other agents during the experiment period over all agents.} {As an example, in a scenario consisting of $100$ agents, if $15$ agents undergo a collision at any time step, then the safety rate will be $85\%$.} { The goal \textbf{reaching rate}, or simply, the reaching rate, is defined as the ratio of agents that reach their goal location by the end of the experiment period\footnote{Note that we do not consider collision dynamics in our experiments and the agents continue moving after a collision.}.} The \textbf{success rate} is defined as the ratio of agents that are safe and reach their goals. 
For each environment, we evaluate the performance over $32$ instances of randomly chosen initial and goal locations for $3$ policies trained with different random seeds. We report the mean rates and their standard deviations for the $32$ instances for each of the $3$ policies (i.e., average performance over $96$ experiments). 

\subsection{Implementation details}\label{sec:implmnt details}

Our learning framework contains two neural network models: the GCBF $h_\theta$ and the controller $\pi_\phi$. The sizes of the MLP layers in them are shown in \Cref{tab:mlp-size}. {The resulting trained control policy is loaded on each agent locally and requires $\SI{1.5}{\mega\byte}$ of memory.}
To make the training easier, we define $\pi_\phi$ = $\pi^\mathrm{NN}_\phi + \pinom$, where $\pi^\mathrm{NN}_\phi$ is the NN controller, so that $\pi^\mathrm{NN}_\phi$ only needs to learn the deviation from $\pinom$. 

\begin{table}[h]
    \centering
    \caption{Size of layers of the MLP used in the GNN.}
    \begin{tabular}{c|ccc}
     \toprule
        MLP & Hidden layer size & Output layer size\\
        \midrule
        $\psi_{\theta_1}$ & $256\times256$ & $128$ \\
        $\psi_{\theta_2}$ & $128\times128$ &  $1$ \\
        $\psi_{\theta_3}$ & $256\times256$ & $128$ \\
        $\psi_{\theta_4}$ & $256\times256$ & $1$ for $h_\theta$ and $m$ for $\pi_\phi$ \\
        \bottomrule
\end{tabular}
    \label{tab:mlp-size}
\end{table}

We use Adam~\cite{kingma2014adam} to optimize the NNs for $1000$ steps in training. The training time is around $3$ hours on a 13th Gen Intel(R) Core(TM) i7-13700KF CPU @ $3400\mathrm{MHz}$ and an NVIDIA RTX 3090 GPU. 
We choose the hyper-parameters following \Cref{tab: hyperparams}, where "lr cbf" and "lr policy" denote the learning rate for $h_\theta$ and $\pi_\phi$, respectively. We set $\alpha=1.0$, $\gamma=0.02$, and $\eta_\mathrm{deriv}=0.2$ for all the environments.

\begin{table}[h]
    \centering
    \caption{Hyper-parameters used in our training}
    \label{tab: hyperparams}
    \begin{tabular}{c|cccccc}
    \toprule
        Environment & $T$ & $\eta_\mathrm{ctrl}$ & lr policy & lr GCBF \\
        \midrule
        SingleIntegrator & $1$ & $10^{-4}$ & $10^{-5}$ & $10^{-5}$ \\
        DoubleIntegrator & $32$ & $10^{-4}$ & $10^{-5}$ & $10^{-5}$ \\
        DubinsCar & $32$ & $10^{-5}$ & $3\times10^{-5}$ & $3\times10^{-5}$ \\
        LinearDrone & $32$ & $10^{-3}$ & $10^{-5}$ & $10^{-5}$ \\
        CrazyflieDrone  & $32$ & $3\times10^{-5}$ & $10^{-5}$ & $10^{-4}$ \\
        \bottomrule
    \end{tabular}
\end{table}

\subsection{Baseline methods}\label{sec: baselines}
We compare GCBF+ with GCBFv0 \cite{zhang2023distributed}, InforMARL \cite{nayak2023scalable}, {MPC \cite{sathya2018embedded},} and centralized and distributed variants of hand-crafted CBF-QPs \cite{wang2017safety}. 
We use a modified version of the GCBFv0 introduced in \cite{zhang2023distributed} where we remove the online policy refinement step since it requires multiple rounds of inter-agent communications to exchange control inputs during execution and does not work in the presence of actuation limits.

The InforMARL algorithm is a variant of MAPPO~\cite{yu2022surprising} that uses a GNN architecture for the actor and critic networks.
We use a reward function that consists of three terms. First, we penalize deviations from the nominal controller, i.e.,
{\begin{equation}
    R_{\mathrm{nom},i} \coloneqq -\frac{1}{2} \norm{ u_i - u_i^{\mathrm{nom}}}^2.
\end{equation}}
To improve performance, we use a sparse reward term for reaching the goal, i.e.,
\begin{equation} \label{eq:sparse_goal_rew}
    R_{\mathrm{goal},i} \coloneqq \begin{dcases}1.0, & \norm{p_i - p^g_i} \leq 2r \\ 0 & \mathrm{otherwise}\end{dcases}
\end{equation}
Safety is incorporated by adding the following term to the reward function to penalize collisions, similar to \cite{semnani2020multi,chen2017decentralized}:
\begin{align}
    R_{\mathrm{col},i} &\coloneqq \max\left\{\max_{j\in\tilde{\mathcal N}_i} R_{\mathrm{col,agent},ij}, \max_{j\in n_\mathrm{rays}} R_{\mathrm{col,obs}, ij}\right\},
\end{align}
where
\begin{align}
    R_{\mathrm{col,agent},ij} &\coloneqq \begin{dcases}
    -1 & \norm{p_i-p_j} < 2r, \\
    \frac{\norm{p_i - p_j}}{2r} - 2 & 2r \leq \norm{p_i - p_j} \leq 4r, \\
    0 & 4r < \norm{p_i - p_j},
    \end{dcases}
\end{align}
for inter-agent collisions, and
\begin{align}
    R_{\mathrm{col,obs},ij} &\coloneqq \begin{dcases}
    -1 & \norm{p_i-p_j^{(i)}} < r, \\
    \frac{\norm{p_i - p_j^{(i)}}}{2r} - 2 & r \leq \norm{p_i - p_j^{(i)}} \leq 3r, \\
    0 & 4r < \norm{p_i - p_j^{(i)}}.
    \end{dcases}
\end{align}
for agent-obstacle collisions.
The final reward function is a sum of the above terms weighted by $\lambda_{\text{nom}}$, $\lambda_{\text{goal}}$ and $\lambda_{\text{col}} > 0$.
{
\begin{equation}\label{eq:Informarl reward}
    R_{i} \coloneqq \lambda_{\text{nom}} R_{\text{nom},i} + \lambda_{\text{goal}} R_{\text{goal},i} + \lambda_{\text{col}} R_{\text{col},i}
\end{equation}
}

{
We use a distributed MPC that does not assume inter-agent communication \cite{sathya2018embedded}, similar to the other baselines.
At each time step, since the control actions of neighbor agents are unknown, they are assumed to follow a constant velocity model, and each agent solves the following $H$-horizon optimal control problem:
\begin{subequations}\label{eq:mpc_baseline}
\begin{align}
    \min_{u_i} \quad & \|u_i- u_i^{\mathrm{nom}}\|^2, \\
    \text{s.t.} \quad &  u_i \in \mathcal{U}, \\
& x_i^k \in S_{N,i}, \quad k = 0, \dots, H-1 \label{eq:constr:mpc_safety}
\end{align}
\end{subequations}
where we use a constant velocity prediction of other agents in the definition of the safe set $S_{N,i}$ in \eqref{eq:constr:mpc_safety}.
In the SingleIntegrator environment, we estimate the current velocity using the past two position observations, while in other environments, velocity is included in agents' states. The MPC baseline is implemented in CasADi \cite{andersson2019casadi} using the SNOPT \cite{gill2005snopt} optimizer.
}

For the hand-crafted CBF-QPs, we define a pairwise higher-order CBF \cite{nguyen2016exponential} between each pair of agents $(i. j)$\footnote{Except for the single integrator environment, where we use the same $h_0$ as in \eqref{eq:h0_distance} and define $h \coloneqq h_0$. For the double integrator environment, \cite{wang2017safety} proposes a CBF that considers the input constraints, which we compare with in \Cref{app: cbf-input-constraint}.}:
\begin{align}
    h_{0} & =\sum_{l\in \mathbb P}(p_i^l-p_j^l)^2- (2r)^2 \label{eq:h0_distance} \\
    h  & = \dot h_{0} + \alpha_{0} h_{0},
\end{align}
where $\mathbb P = \{x, y\}$ for 2D environments and $\mathbb P = \{x, y, z\}$ for 3D, $\alpha_{0}\in \mathbb R_+$ is a constant and $r$ is the radius of the agent. 
In the LiDAR environments, we also use a pairwise CBF between each agent $i$ and its LiDAR hitting points $j$, defined as
\begin{align}
    h'_{0} & =\sum_{l\in \mathbb P}\left(p_i^l-p_j^{(i),l}\right)^2- r^2 \\
    h'  & = \dot h'_{0} + \alpha_{0} h'_{0},
\end{align}
where $p_j^{(i),l}$ is the $l$-th position dimension of LiDAR hitting point $p_j^{(i)}$.
We consider two CBF-QP frameworks \cite{wang2017safety}:
\begin{itemize}
    \item Centralized CBF: In this framework, inputs of all the agents are solved together by setting up a centralized QP containing CBF constraints of all the agents. In this case, the CBF condition is
    \begin{align}\label{eq: hc CBF condition}
        \dot h  + \alpha h  \geq 0, \quad \dot h' + \alpha h' \geq 0.
    \end{align}
    \item Decentralized CBF: In this framework, each agent computes its control input but the CBF condition is \textit{shared} between the neighbors as in \cite{wang2017safety}. Let $\dot x_{ij} = f_{ij}(x_{ij}) + g_{ij, i}(x_{ij})u_{i} + g_{ij, j}(x_{ij})u_{j}$ denote the combined dynamics of the pair $(i,j)$ where $f_{ij}, g_{ij,i}, g_{ij,j}$ can be obtained from combining the agent dynamics. Then, the constraint used in the agent $i$'s QP is:
    \begin{align}\label{eq: hij i share}
        \frac{\partial h }{\partial x_{ij}}g_{ij, i}(x_{ij})u_i \geq  -\frac{1}{2}\left(\alpha h (x_{ij}) +\frac{\partial h }{\partial x_{ij}}{f_{ij}}(x_{ij})\right),
    \end{align}
    while that in agent $j$'s QP is:
    \begin{align}\label{eq: hij j share}
         \frac{\partial h }{\partial x_{ij}}g_{ij, j}(x_{ij})u_j \geq  -\frac{1}{2}\left(\alpha h (x_{ij}) +\frac{\partial h }{\partial x_{ij}}{f_{ij}}(x_{ij})\right),
    \end{align}
    so that the sum of the constraints \eqref{eq: hij i share} and \eqref{eq: hij j share} recovers the CBF condition \eqref{eq: hc CBF condition}. Since the obstacles are not controlled, for the agent-LiDAR pair, the constraint used in agent $i$'s QP is:
    \begin{align}
         \frac{\partial h' }{\partial x_{ij}}g_{ij, j}(x_{ij})u_j \geq  -\left(\alpha h' (x_{ij}) +\frac{\partial h' }{\partial x_{ij}}{f_{ij}}(x_{ij})\right).
    \end{align}
\end{itemize}
For both centralized and decentralized approaches, we design two baselines with $\alpha=1.0$ and $\alpha=0.1$, respectively. The resulting $4$ baselines are named CBF1.0, CBF0.1, DecCBF1.0, and DecCBF0.1, respectively. 

\begin{figure*}
    \centering
    \begin{subfigure}{\linewidth}
        \centering
        \includegraphics[width=\linewidth]{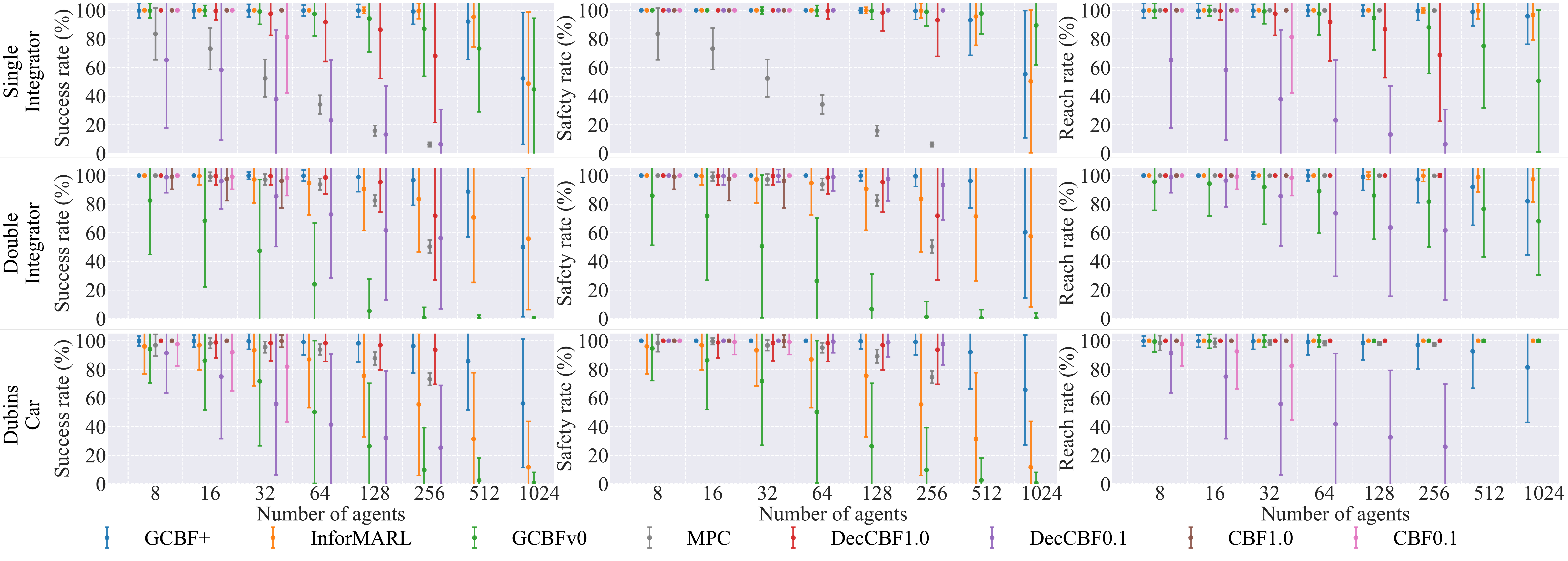}
        \caption{2D environments}
    \end{subfigure}

    \vspace{1.8em}
    \begin{subfigure}{\linewidth}
        \centering
        \includegraphics[width=\linewidth]{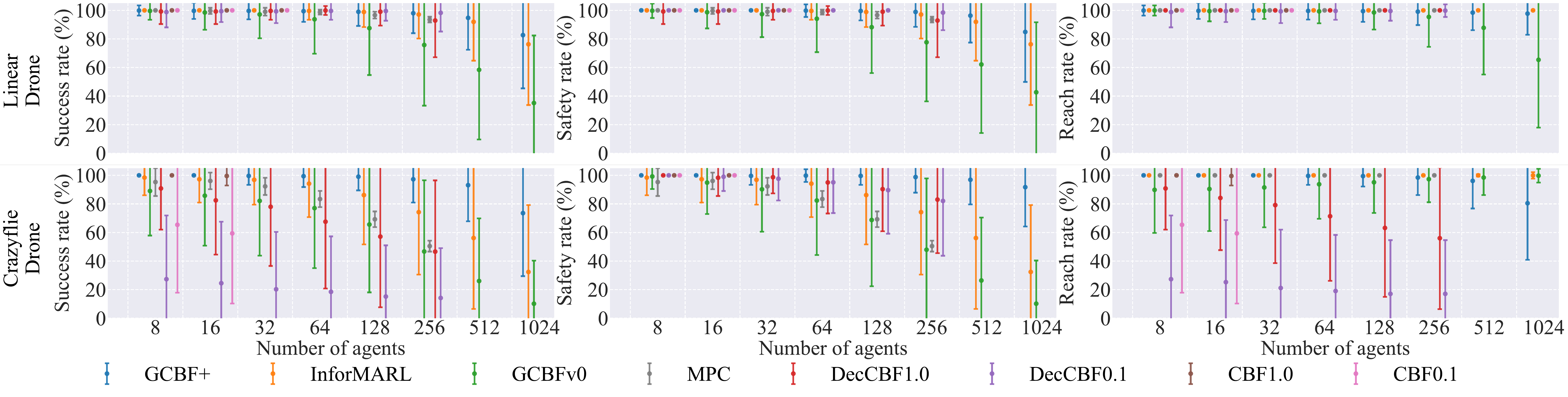}
        \caption{3D environments}
    \end{subfigure}
    \caption{Success (left), safety (middle), and reach (right) rates for an increasing number of agents for fixed area width $l$.}
    \label{fig: 2D3D rates}
\end{figure*}

\section{Numerical experiments: results}\label{sec:experiments}
We conduct simulation experiments to examine the scalability, generalizability, and effectiveness of the proposed method. In all experiments, the initial position of the agents and goals are uniformly sampled from the set $\mathbb{P}_0 = [0, l]^\mathfrak{n}$ for an area width $l > 0$ which we specify for each environment. The \textit{density} of agents can be approximately computed as $N / l^\mathfrak{n}$. Hence, a smaller value of $l$ results in a higher density of agents and thus is more challenging to prevent collisions. We train GCBF+ and GCBFv0 on an environment with $8$ agents and $8$ obstacles with $l=4$ for 2D and $l=2$ for 3D environments.

\subsection{Performance under increasing number of agents} \label{subsec:exp_increasing_agent}
We first perform experiments in an obstacle-free setting where we test the algorithms for a fixed $l$ but increase the number of agents $N$ from $8$ to $1024$. This tests the ability of each algorithm to maintain safety as the density of agents and goals in the environment increases by over $100$-fold. We use $l = 8$ for SingleIntegrator and DoubleIntegrator environments, $l=16$ for the DubinsCar environment, and $l=4$ for both the 3D environments\footnote{
Note that we use more challenging (i.e., smaller) values of $l$ for the 2D ($8, 16$ vs $32$) and 3D ($4$ vs $16$) as compared to our previous work \cite{zhang2023distributed}.
} and show the resulting success rate, safety rate and reach rate in \Cref{fig: 2D3D rates}.

\textbf{Centralized methods do not scale with increasing number of agents. }
As expected, the centralized methods (i.e., CBF1.0 and CBF0.1) require increasing amounts of computation time as the number of agents increases. Consequently, we were unable to test CBF1.0 and CBF0.1 for more than $32$ agents in all environments due to exceeding computation limits.

\textbf{{Hand-crafted CBFs are overly conservative. }}
The decentralized {hand-crafted} CBF-QP with $\alpha=0.1$ has comparable safety performance to GCBF+. However, it is much more conservative than GCBF+ and compromises its goal-reaching ability, as evident from the low reach rates in all environments. For larger $\alpha$, the decentralized {hand-crafted} CBF-QP method fails to maintain a high safety rate as the controls become saturated by the control limit. Although the decentralized {hand-crafted} CBF can be scaled to a large number of agents assuming that each agent can perform computation for its control input locally, in our experiments, we simulate the decentralized controller on one computer node and hence are constrained by the memory and computation limits of the computer. Thus, we could not perform experiments for more than $256$ agents. However, the downward trend of the reaching rate illustrates that the decentralized {hand-crafted} CBF-QP method becomes more conservative as the environment gets denser.

{
\textbf{MPC has low safety rates because of the non-cooperative agents. }
MPC's success rate drops significantly in all environments as the number of agents increases because the agents do not know the \textit{actual} actions of other agents and only make predictions from observations.
This problem is exacerbated in
the SingleIntegrator environment where the agent velocities can change instantaneously. This leads to the predicted motions of neighboring agents differing significantly from the actual actions, resulting in low safety rates. 
On the other hand, for environments with higher relative degrees, the agent velocities do not change instantaneously. Consequently, the constant-velocity model yields relatively more accurate predictions, and hence, improved safety rates. However, the performance of MPC is still considerably poorer than that of GCBF+ in all experiments. This is due to the fact that, while GCBF+ also does not have access to the control inputs of other agents during \textit{online execution}, however, during \textit{offline centralized training}, GCBF+ policy has access to the control inputs of other agents which also come from the same GCBF policy being used by the different agents. As a result, the GCBF+ policy can ensure that the GCBF conditions are satisfied.
Then, during online execution, the GCBF condition is satisfied if all other agents follow the same GCBF policy that was used during offline training \textit{without} the need for other agents to communicate their control inputs.
While distributed MPC can be scaled to more than $256$ agents, the computation is carried out on a single computer in our experiments, and the computational requirement of MPC exceeds the capability of one computer. Therefore, we do not test MPC for more than $256$ agents. Lastly, since MPC gets even slower with LiDAR data, we do not use MPC as a baseline when conducting experiments with obstacles in \Cref{sec: static-obs}.
}

\begin{figure*}
    \centering
    \includegraphics[width=\textwidth]{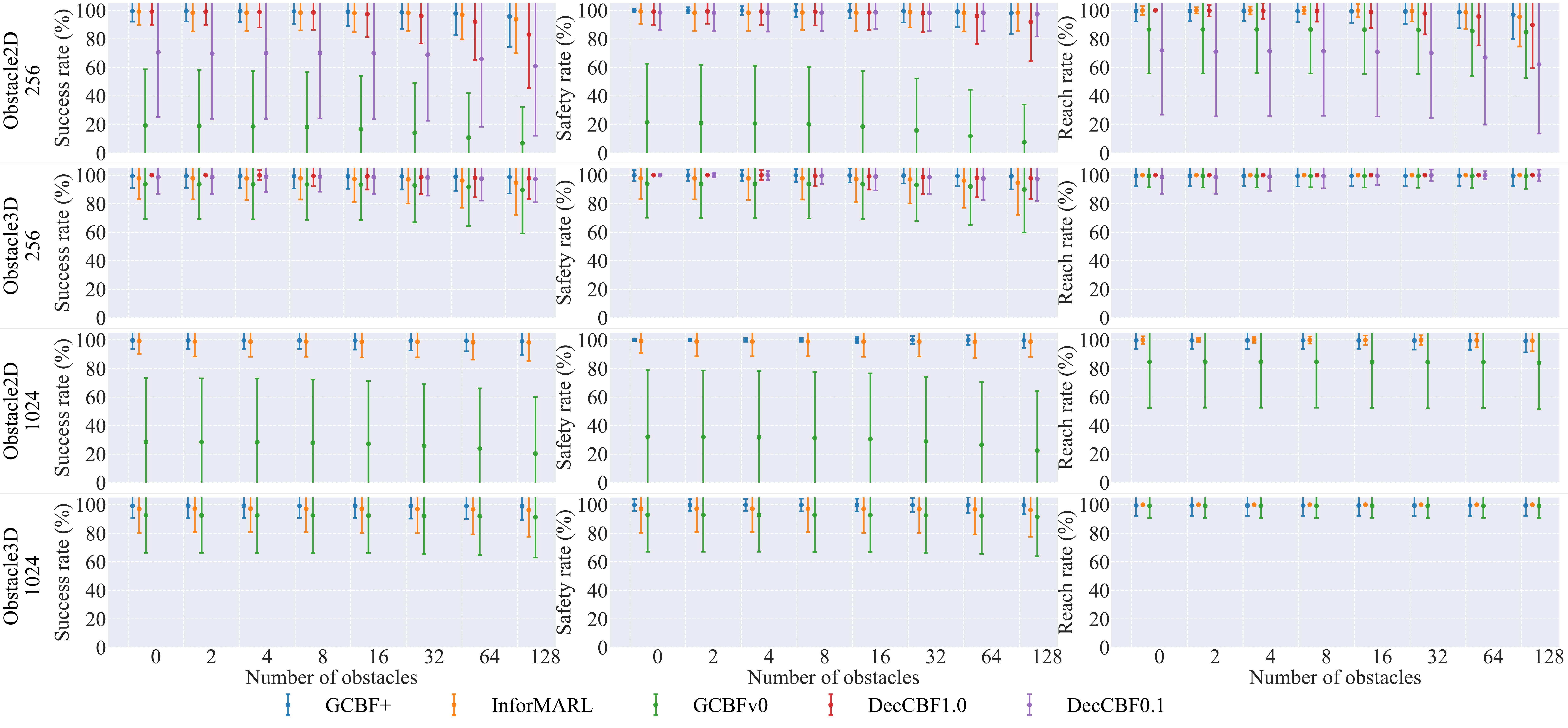}
    \caption{\textbf{Obstacle environment performance}: 
    Success (left), safety (middle), and reach (right) rates for the obstacle environment in one 2D and one 3D environment, namely, DoubleIntegrator and LinearDrone environments. In these experiments, the number of agents as well as the area size are kept constant while the number of obstacles are increased. The first set of experiments is conducted on 256 agents with all the baselines. The second set of experiments is conducted with 1024 agents
    with only GCBF and InforMARL for comparison since we were unable to simulate other baselines for more than 256 agents due to computational limits.}
    \label{fig: obs cases}
\end{figure*}

\textbf{GCBFv0 struggles with safety for dynamics with high relative degrees. }
While GCBFv0 performs comparably on the SingleIntegrator environment, the performance deteriorates drastically on all other dynamics. This is because GCBFv0 relies on an accurate hand-crafted safe control invariant set during training, which is difficult to estimate in the presence of control limits for dynamics with high relative degrees. The safe control invariant set is easy to estimate for relative degree $1$ environments such as SingleIntegrator, where it can be taken as the complement of the unsafe set. However, for high relative degree dynamics with control limits, the naive estimation method used by GCBFv0 breaks down, causing the safety rate and thus success rate to drop significantly. Another potential reason for the poor safety of GCBFv0 is that it uses $\pinom$ in the control loss, which forces a trade-off between safety and goal-reaching (see \Cref{remark:new_ctrl_loss}).

\textbf{GCBF+ performs well on nonlinear dynamics. }
We observe that all methods have lower success rates in environments that have nonlinear dynamics (DubinsCar, CrazyflieDrone) compared to ones with linear dynamics (SingleIntegrator, DoubleIntegrator, LinearDrone). The performance gap between GCBF+ and other methods is more clear in these challenging environments. On the DubinsCar environment, GCBF+ achieves a $44\%$ higher (compared to InforMARL) and $55\%$ higher (compared to GCBFv0) success rate. On the CrazyflieDrone environment, GCBF+ achieves a $45\%$ higher (compared to InforMARL) and $65\%$ higher (compared to GCBFv0) success rate. Hence, GCBF+ generalizes better than the baseline algorithms, particularly for environments with nonlinear dynamics.

\textbf{GCBF+ reach rate declines faster than InforMARL and MPC. }
While the safety rate for GCBF+ is the best among the baselines for denser environments, its reach rate declines as the number of agents increases, while the reach rate for InforMARL {and MPC} stays consistently near $100\%$ in all environments. The main reason for this decline is that GCBF+ focuses on safety and delegates the liveness (i.e., goal-reaching) requirements to the nominal controller which is unable to resolve deadlocks. Hence, one potential reason for the lower reach rates of GCBF+ as the density increases is that the learned controller is unable to resolve deadlocks that occur more frequently with increasing density. On the other hand, InforMARL has a sparse reward term for reaching the goal \eqref{eq:sparse_goal_rew} and hence, it is incentivized to learn a controller that can resolve deadlocks at the cost of temporarily deviating from the nominal controller and sacrificing safety, which is evident from the significant drop in the safety rate for InforMARL. 

\begin{figure*}
    \centering
    \includegraphics[width=0.48\textwidth]{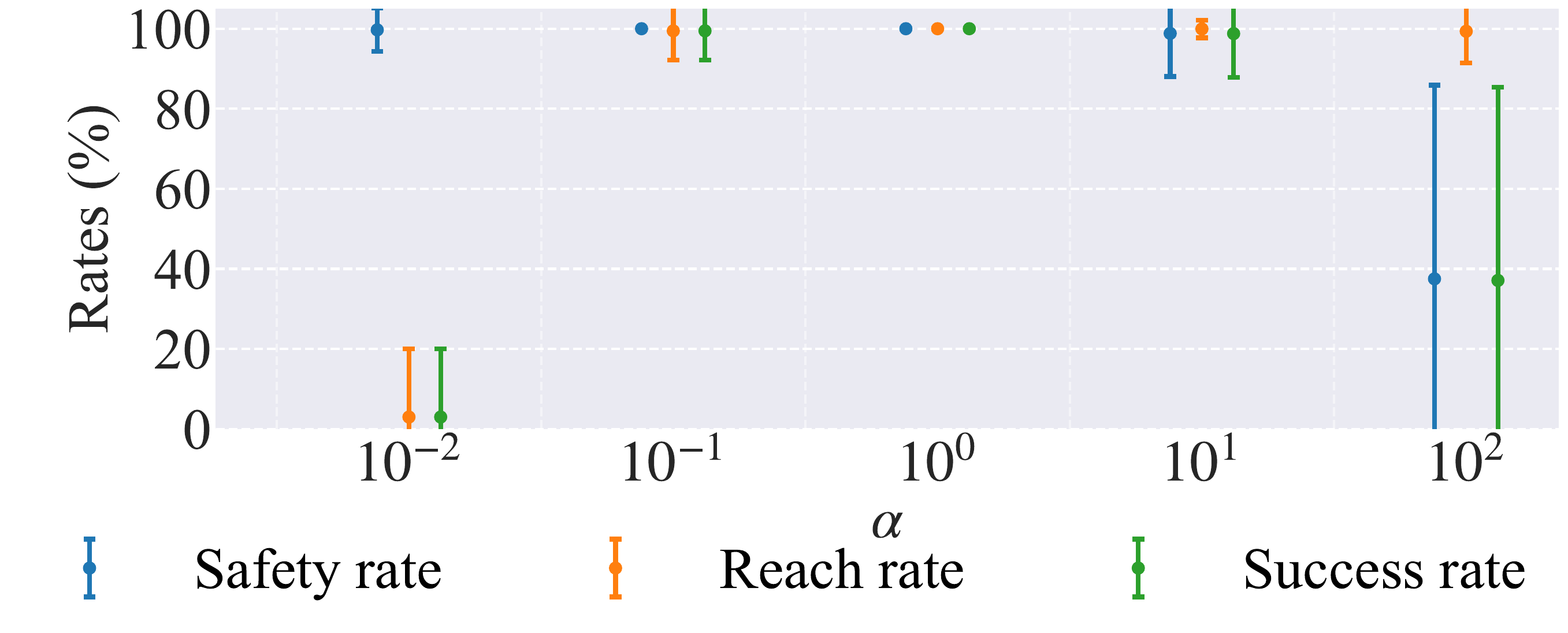}
    \includegraphics[width=0.48\textwidth]{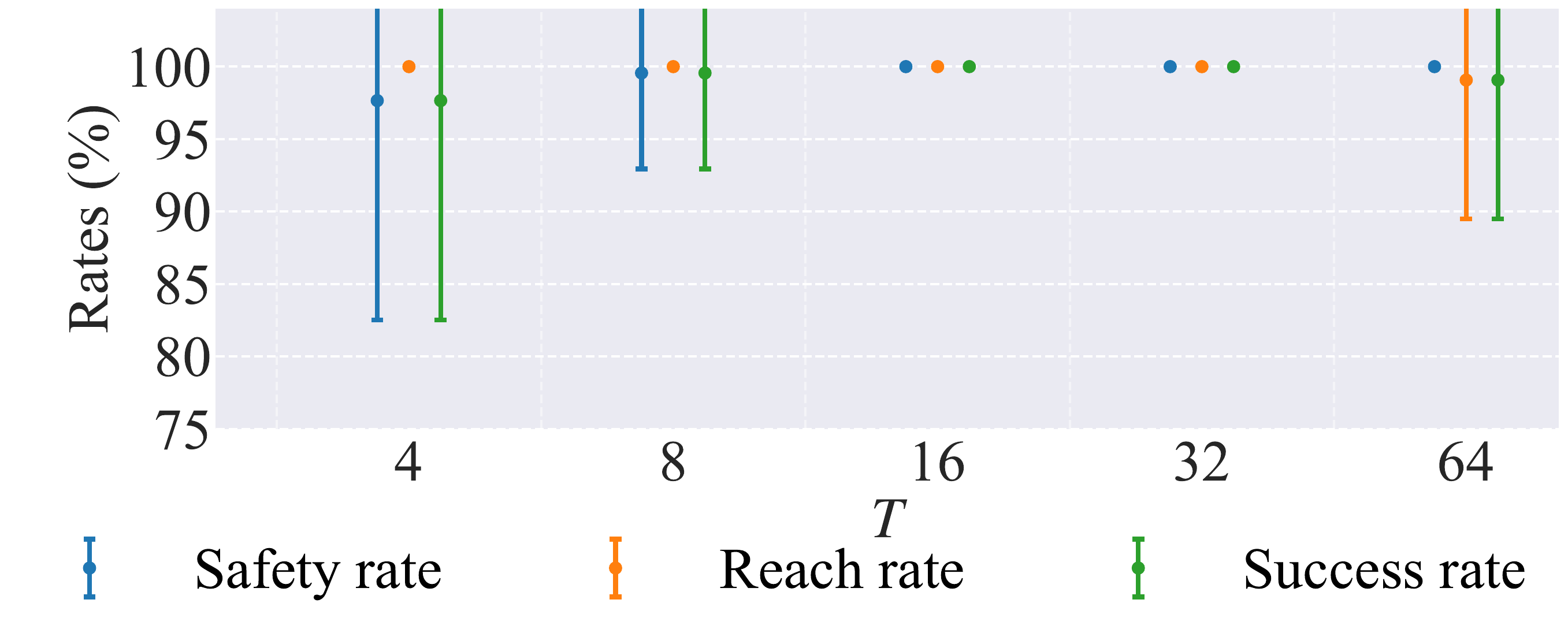}
    \caption{\textbf{Left}: Sensitivity of GCBF+ to the CBF class$-\mathcal K$ function parameter $\alpha$. \textbf{Right}: Sensitivity of GCBF+ to prediction horizon $T$ for computation of the control invariant set $\mathcal S_{c, i}$.}
    \label{fig:sens}
\end{figure*}

\subsection{Performance under increasing number of obstacles}\label{sec: static-obs}
In the next set of simulation experiments, we fix the number of agents $N$ and the area width $l$ and vary the number of obstacles present from $0$ to $128$. For the 2D DoubleIntegrator environment, we consider $(N=256, l=16)$ and $(N=1024, l=32)$, where the obstacles are cuboids with side lengths uniformly sampled from $[0.1, 0.5]$ and each agent generates $32$ equally spaced LiDAR rays to detect obstacles. For the 3D LinearDrone environment, we consider $(N=256, l=8)$ and $(N=1024, l=12)$, where the obstacles are spheres with radius uniformly sampled from $[0.15, 0.3]$ and each agent generates $130$ equally spaced LiDAR rays to detect obstacles. 

The success rate, safety rate, and reach rate for all cases are shown in \Cref{fig: obs cases}.
Overall, we observe similar trends as the previous experiment in \Cref{subsec:exp_increasing_agent}. In all environments, GCBF+ has the highest success rates compared with the baselines.
Trained with just $8$ agents and $8$ obstacles, GCBF+ can achieve a $>98\%$ success rate with $256$ (and $1024$) agents and $128$ obstacles.
InforMARL performs well but is behind GCBF+. Other baselines have much lower success rates compared with GCBF+ and InforMARL. GCBFv0 does not perform well since it does not account for control limits. The decentralized {hand-crafted} CBF-QPs perform poorly in the 2D environment due to their conservatism and in the 3D environment due to saturation from the control limits.

\subsection{Sensitivity to hyper-parameters} \label{subsec:sensitivity}

We next perform a sensitivity analysis of our proposed method on the DoubleIntegrator environment to investigate the effect of two hyper-parameters: $\alpha$ and $T$. The $\alpha$ parameter is used to define the CBF derivative condition \eqref{eq:L_cbf_i}, while $T$ is used to label safe control invariant data and unsafe data (see \Cref{sec:training-data}). We plot the success, reach, and safety rates while varying $\alpha$ from $10^{-2}$ to $10^2$ in the left plot in \Cref{fig:sens}. The results showed that using $\alpha=10^{-2}$ led to a drop in the reach rate, while using $\alpha=10^{2}$ led to a drop in the safety rate. This behavior can be attributed to the fact that for very small values of $\alpha$, the CBF condition becomes overly conservative, resulting in poor goal-reaching. For very large values of $\alpha$, safety can be compromised as the CBF condition allows the system to move towards the unsafe set at a faster rate. This, along with the fact that the control inputs are constrained, may lead to a violation of safety. 
Note that for 
$\alpha\in [0.1, 10]$, the performance of GCBF+ does not change much. This implies that GCBF+ is robust to a large range of $\alpha$. 

\begin{figure}
    \centering
    \includegraphics[width=1\columnwidth]{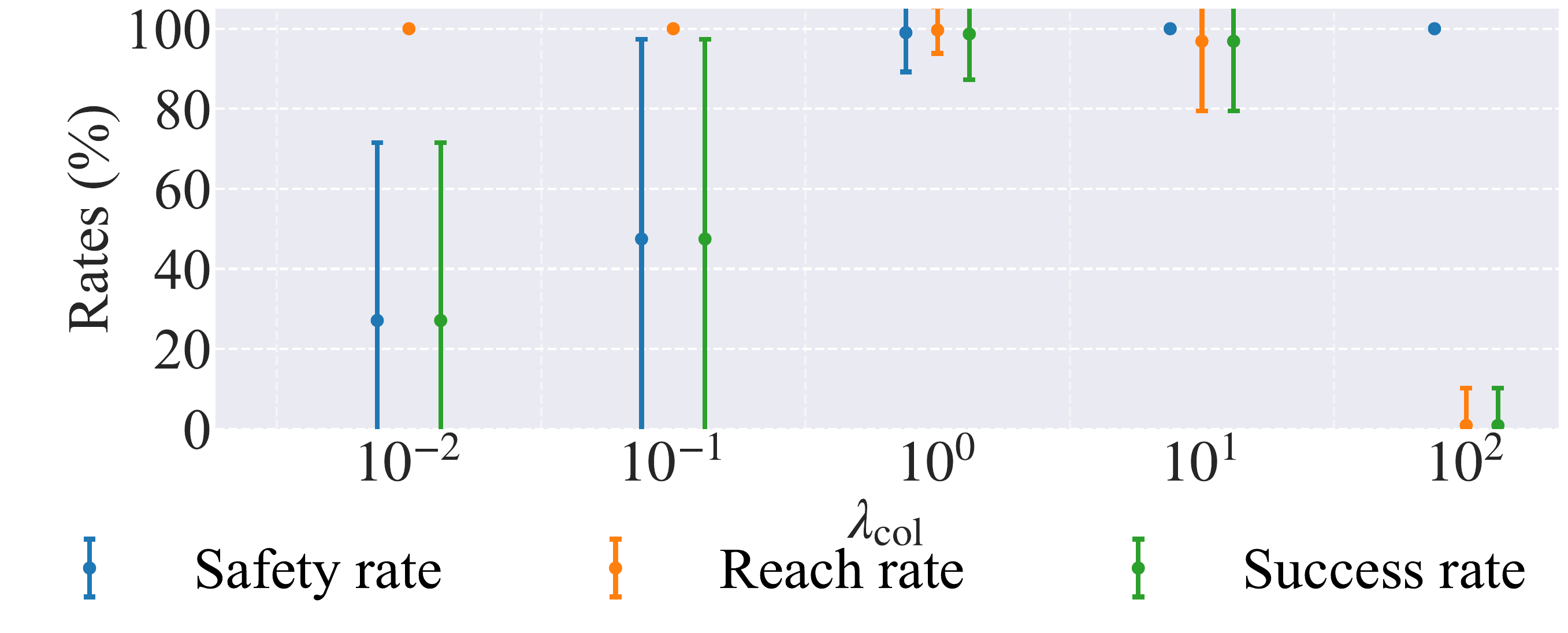}
    \caption{Sensitivity of InforMARL to the weight $\lambda_\mathrm{col}$ in \eqref{eq:Informarl reward}.}
    \label{fig:info sens}
\end{figure}

\begin{figure*}
    \centering
    \begin{subfigure}{0.24\linewidth}
        \centering
        \includegraphics[width=\linewidth]{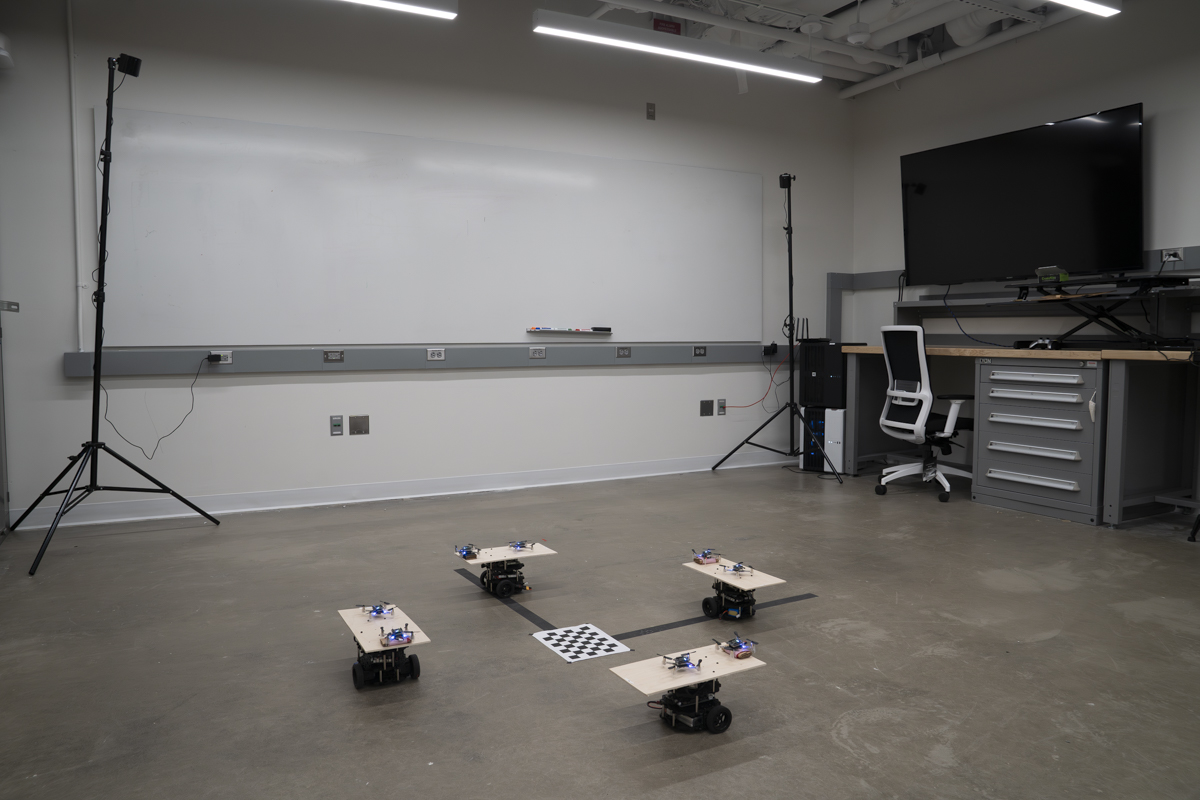}
        \caption{}
    \end{subfigure}
    \begin{subfigure}{0.24\linewidth}
        \centering
        \includegraphics[width=\linewidth]{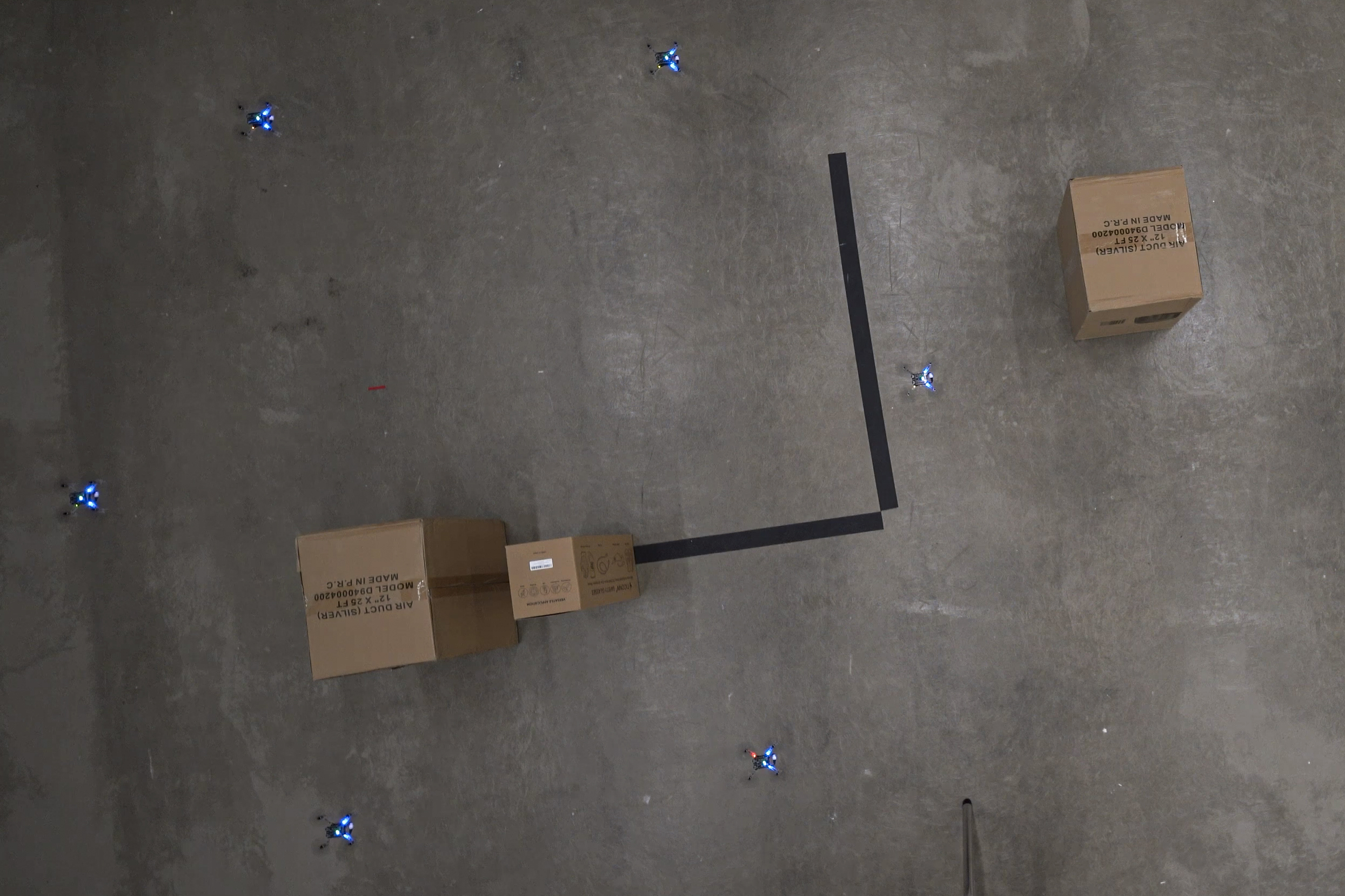}
        \caption{}
    \end{subfigure}
    \begin{subfigure}{0.24\linewidth}
        \centering
        \includegraphics[width=\linewidth]{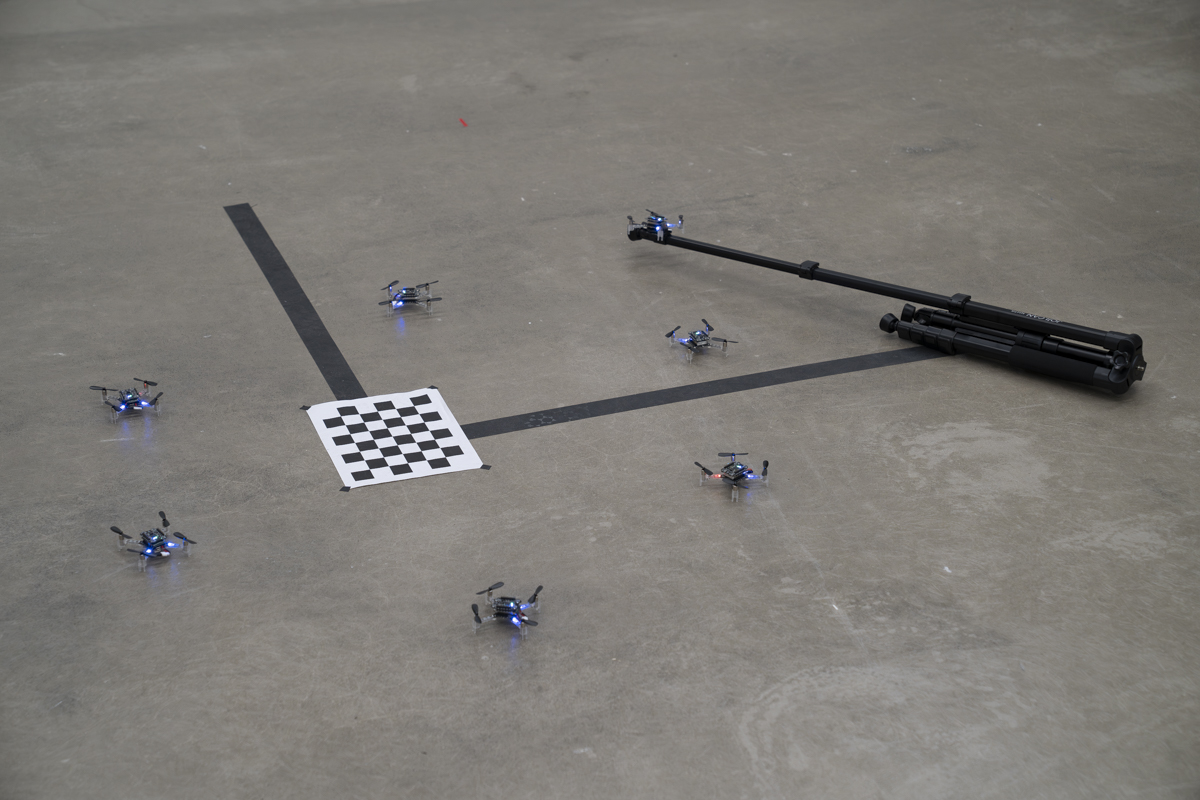}
        \caption{}
    \end{subfigure}
    \begin{subfigure}{0.24\linewidth}
        \centering
        \includegraphics[width=\linewidth]{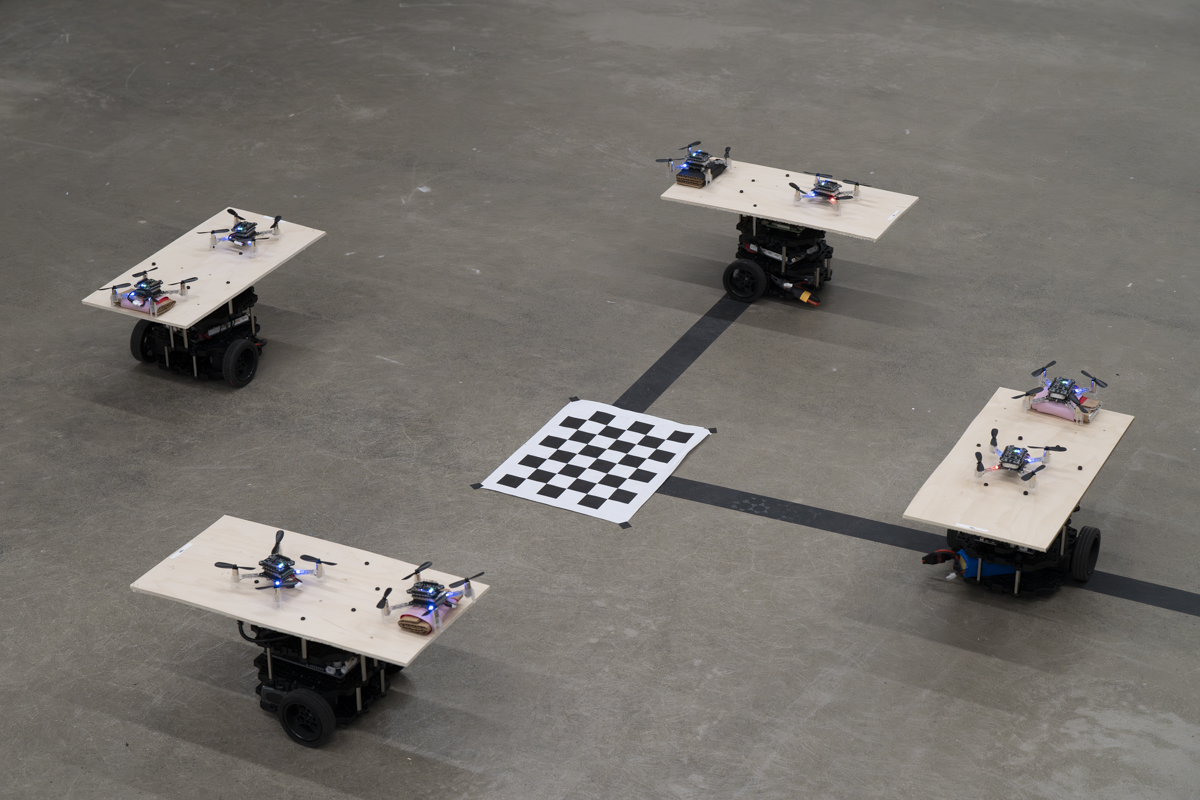}
        \caption{}
    \end{subfigure}
    \caption{\textbf{Hardware experiment setup: } (a) The overall experimental setup with ground robots, CF drones and the motion capture system.
    {(b) Setup for the position exchange with large static obstacles.}
    (c) {Setup for the position exchange with a drone acting as a moving obstacle, which is mounted on a tripod to be moved around randomly.}
    (d) Setup for tracking and landing on a moving platform. 
    }
    \label{fig:CF exp setup}
\end{figure*}

\begin{figure}
    \centering
    \includegraphics[width=1\linewidth]{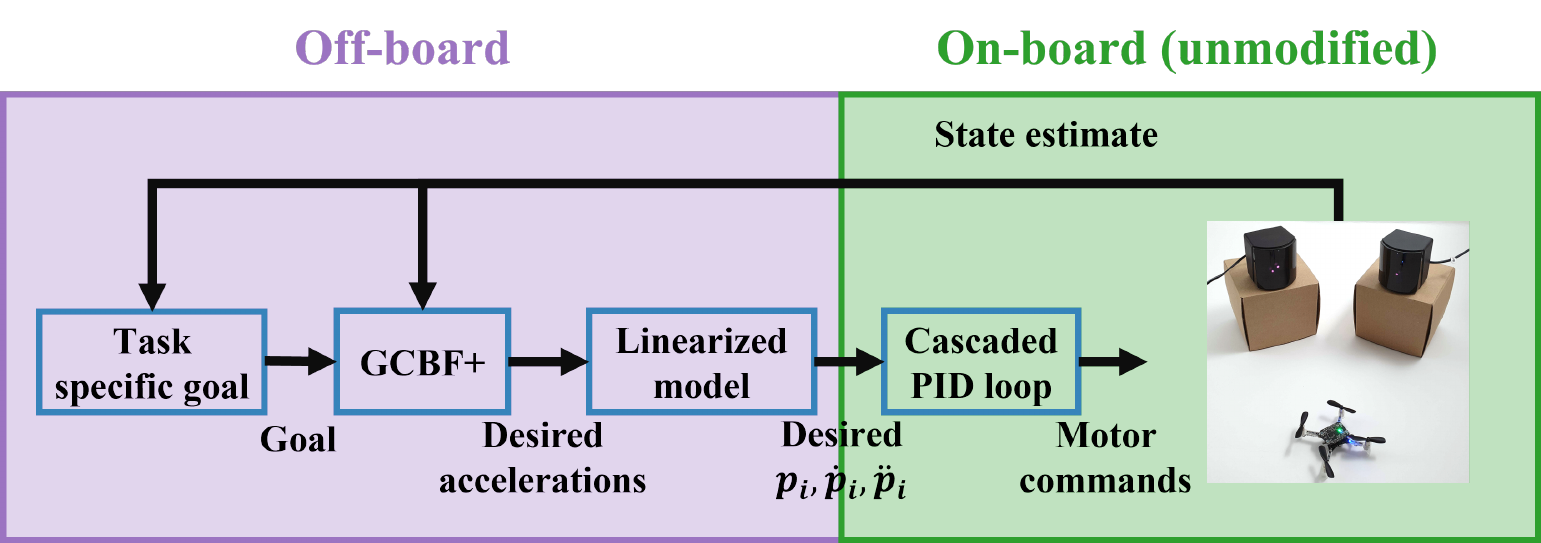}
    \caption{Hardware Control Architecture for Crazyflies}
    \label{fig:arch}
\end{figure}

The right plot in \Cref{fig:sens} analyzes the performance of GCBF+ for varying prediction horizons $T\in [4, 64]$ for labeling the data to be safe control invariant or unsafe for training. For a very small horizon $T = 4$, the safety rate drops as the resulting approximation of the safe control invariant set is poor. For a very large horizon $T = 64$, the algorithm becomes too conservative, requiring longer training times to converge. For the chosen fixed number of training steps, we observe that the resulting controller, while maintaining $100\%$ safety, achieves $98\%$ goal-reaching rate. However, as observed from the plots, GCBF+ is mildly sensitive to this parameter only at its extreme values, and almost insensitive in the nominal range $T\in [8, 32]$. 

As InforMARL has the best performance among baselines, we analyze its sensitivity as well. \Cref{fig:info sens} analyzes the sensitivity of the performance of InforMARL to the weight $\lambda_\mathrm{col}$ that dictates the penalty for collision in the RL reward function \eqref{eq:Informarl reward}. It can be observed that for a relatively small range of $\lambda_\mathrm{col}\in [1, 10]$, InforMARL achieves high performance. For smaller values of this weight, the RL-based method over-prioritizes goal-reaching, compromising on safety, and for larger values of this weight, the goal-reaching performance is poor due to over-prioritization of the system safety. 

These experiments illustrate that the proposed algorithm GCBF+ is not as sensitive to its crucial hyper-parameters as InforMARL, and hence, does not require fine-tuning of such parameters to obtain desirable results.

\begin{figure*}
    \centering
    \includegraphics[width=0.99\linewidth]{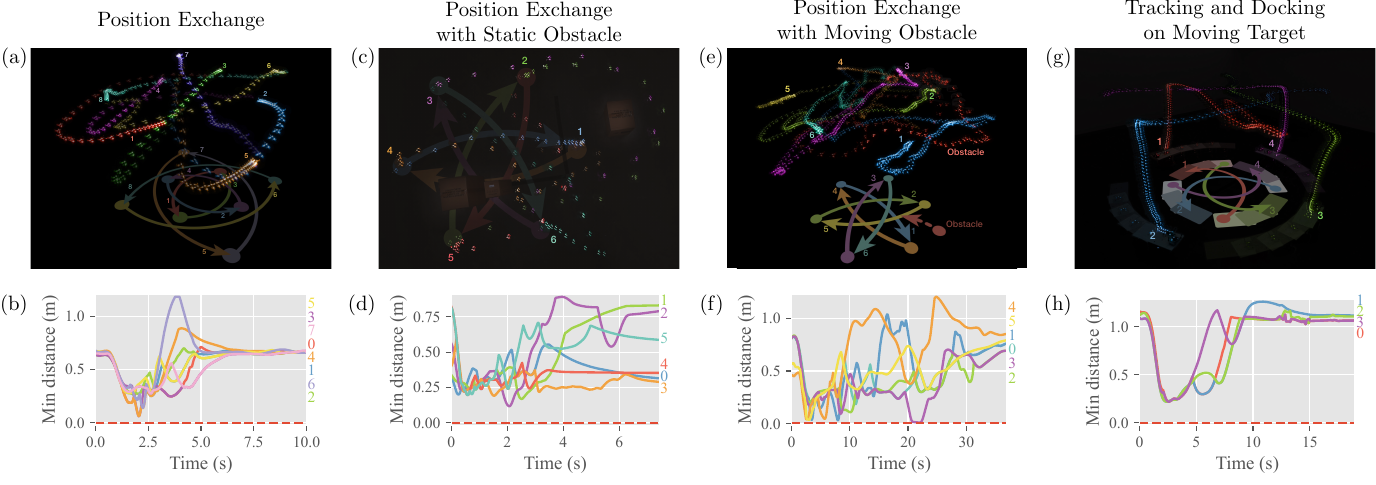}
    \caption{\textbf{Hardware Experiment Results: } (a,c,e,g) Time-lapse illustration of a Crazyflie swarm controlled via GCBF+. (b,d,f,h) The associated minimum distances for each agent to other agents and obstacles.}
    \label{fig:hw_all}
\end{figure*}

\section{Hardware experiments}\label{sec:hardware exps}
We conduct hardware experiments on a swarm of Crazyflie (CF) 2.1 platform\footnote{https://www.bitcraze.io/products/crazyflie-2-1/} to illustrate the applicability of GCBF+ on real robotic systems. We conduct four sets of experiments as discussed below. The hardware setup is illustrated in \Cref{fig:CF exp setup}. To communicate with the CFs, two Crazyradios are used. Localization is performed using the {Lighthouse localization system}\footnote{http://tinyurl.com/CFlighthouse}. Four {SteamVR Base Station 2.0}\footnote{http://tinyurl.com/lighthouseV2S} are mounted on tripods and placed on the corners of the flight area.

\subsection{Control architecture} \label{sec:control_arch}
An overview of the hardware control architecture is shown in
\Cref{fig:arch}.
Computation is split into onboard, i.e., on the CF micro-controller unit (MCU), and off-board, i.e., a laptop connected to the CFs over Crazyradio. 
To communicate with the Crazyflies, we use the crazyswarm2 ROS2 package\footnote{https://github.com/IMRCLab/crazyswarm2}, which allows for receiving full state estimates from and sending control commands to the CFs. We use a single ROS2 node for the off-board computations at 50 Hz.

\textbf{Task Specific Logic}
The state estimates are used to compute a task-specific goal position for each of the CFs. For the swapping tasks, the goals do not change. For the docking task, we take the location of the Turtlebots to be the goal position.

\textbf{GCBF+ Controller}
The goal positions and the state estimates (position and velocities) are next used to compute target accelerations for each CF using the GCBF+ controller. This GCBF+ controller is trained using double integrator dynamics. 

\textbf{Ideal Dynamics Model}
To track the computed desired accelerations from GCBF+, we make use of the {\texttt{cmd\_full\_state}}\footnote{http://tinyurl.com/CFcmdFullS} interface in \verb|crazyswarm2|. However, \verb|cmd_full_state| requires set points for the whole state (i.e., also position and velocity) instead of just accelerations. To resolve this problem, we simulate the ideal dynamics model (i.e., double integrator) used for training GCBF+ and take the resulting future positions and velocities that would result from applying the desired accelerations after some duration $\Delta t$. We used $\Delta t = \mathrm{50 ms}$ in our hardware experiments.

We do not modify the onboard computation. The received full state set points are used as set points for a cascaded PID controller, as described in the {Crazyflie documentation}\footnote{http://tinyurl.com/CFCasCadePID}.

\subsection{Experimental results}
We conduct the following four sets of hardware experiments.

\textbf{Position exchange}~ In this experiment, we arrange the CF drones in a circular configuration with the objective of each drone exchanging position with the diagonally opposite drone. We perform experiments with up to $8$ drones. This is a typical experiment setting used for illustration of the capability of an algorithm to maintain safety where there are many inter-agent interactions. The resulting trajectories of the drones are plotted in \Cref{fig:hw_all}(a-b).
As can be observed from the figure, the CF drones maintain the required safe distance and land safely at their desired location. 

{ To further validate the efficacy of GCBF+ on physical robots, we perform the position exchange experiment with $6$ drones in $32$ trials. The distribution of the minimum distance between CF drones is shown in \Cref{fig: hw_dist}. We can observe that the GCBF+ controller successfully avoids collisions in all trials. }

\begin{figure}
    \centering
    \includegraphics[width=.8\columnwidth]{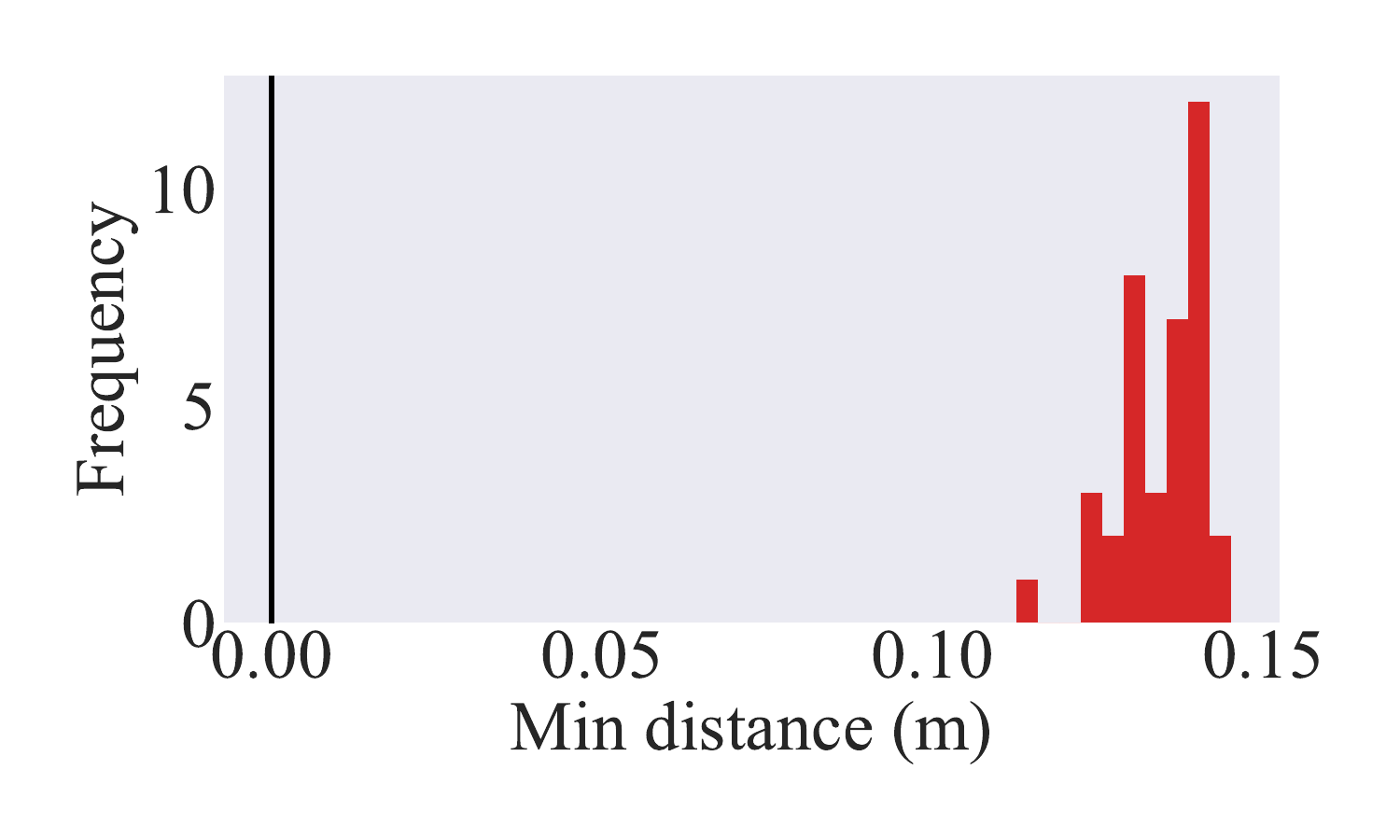}
    \caption{Distribution of the minimum distance between Crazyflies in the position exchange task.}
    \label{fig: hw_dist}
\end{figure}

{\textbf{Position exchange with static obstacles}~ In this experiment, we introduce large obstacles to the position exchange experiment as illustrated in \Cref{fig:CF exp setup}(b). Since the CF drones do not have lidar sensors, the lidar observations are simulated using the current position of the drones.
We perform experiments with $6$ drones, and the experimental results show that the success rate is $100\%$ for all $8$ experiments.}

\textbf{Position exchange with moving obstacle}~ In this experiment, we add a moving obstacle to the setup from the previous experiment { to show the generalization ability of GCBF+ to unseen scenarios}. The moving obstacle is moved arbitrarily around the environment by a human subject. The path of the moving obstacle is not known beforehand by the controlled CF drones. The moving obstacle is the same size as a Crazyflie drone and moves with the same maximum speed as the controlled drones. \Cref{fig:hw_all}(e-f) illustrate that safety is maintained at all times.

\textbf{Tracking and landing on moving target}~ In this experiment, the CF drones are required to track a moving ground vehicle and land on it. For this experiment, instead of a simple LQR controller, a back-stepping controller is used as a nominal controller that can track a moving target with time-varying acceleration. We use four Turtlebot 3 mobile robots\footnote{https://www.turtlebot.com/turtlebot3/} as the moving target, equipped with a platform on its top where a CF drone can land. Initially, one CF drone is placed on each of the Turtlebots. The drones take off and start tracking the diagonally opposite moving target, while the Turtlebots move in a circular trajectory.
From \Cref{fig:hw_all}(g-h), the drones successfully land on the moving targets while maintaining safety with each other in this dynamically changing environment.
This illustrates the generalizability of GCBF+ to a variety of control problems.

\textbf{Experiment videos}~ The experiment videos are available at \href{https://mit-realm.github.io/gcbfplus/}{https://mit-realm.github.io/gcbfplus/}.

\section{Conclusions}\label{sec:conclusions}
In this paper, we introduce a new class of CBF, termed GCBF, to encode inter-agent and obstacle collision avoidance in a large-scale MAS. We propose GCBF+, a training framework that utilizes GNNs for learning a GCBF candidate and a distributed control policy using only local observations. The proposed framework can also incorporate LiDAR point-cloud observations instead of actual obstacle locations, for real-world applications. Numerical experiments illustrate the efficacy of the proposed framework in achieving high safety rates in dense multi-agent problems and its superiority over the baselines for MAS consisting of nonlinear dynamical agents. Trained on 8 agents, GCBF+ achieves over 80\% safety rate in environments with more than 1000 agents, demonstrating its generalizability and scalability. A major advantage of the GCBF+ algorithm is that it does not have a trade-off between safety and performance, as is the case with reinforcement learning (RL)-based methods. Furthermore, hardware experiments demonstrate its applicability to real-world robotic systems.  

The proposed work has a few {limitations}. {In the current framework, there is no cooperation among the controlled agents, which leads to conservative behaviors. In certain scenarios, this can also lead to deadlocks resulting in a lower reaching rate, as observed in the numerical experiments as well. We are currently investigating methods of designing a high-level planner that can resolve such deadlocks and lead to improved performance.} Similar to other NN-based control policies, the proposed method also suffers from difficulty in providing formal guarantees of correctness. In particular, it is difficult, if not impossible, to verify that the proposed algorithm can always keep the system safe via formal verification of the learned neural networks {(see Appendix 1 in \cite{katz2017reluplex} on NP-completeness of NN-verification problem)}. This informs our future line of work on looking into methods of verification of the correctness of the control policy. Lastly, in this work, we assumed that all the agents have the same underlying dynamics for simplicity. For heterogeneous MAS where agent dynamics are different, the type of agents can be encoded in node features.
In addition, edge features can be chosen so that the shared information is the same for different types of nodes. We leave the extension of our method for heterogeneous systems as future work.

\begin{appendices}

\section{Proof of the claim in Remark \ref{remark:R_small}}\label{app: R_small}
For any $i \in V_a$, by definition of $M$ and the continuity of the position of nodes,
changes in the neighboring indices $\mathcal{N}_i$ can only occur without collision at a distance $R$. To see this, we consider the following three cases.

\noindent\textbf{Case 1: } $\abs{\tilde{\mathcal{N}}_i} < M$, i.e., the number of neighbors is less than $M$. In this case, $\mathcal N_i = \Tilde{\mathcal N}_i$ and hence, a node $j$ is added or removed to $\mathcal{N}_i = \tilde{\mathcal{N}}_i$ when it enters or leaves the sensing radius $R$.

\noindent\textbf{Case 2: } $\abs{\tilde{\mathcal{N}}_i} > M$, i.e., the number of neighbors is more than $M$. In this case, there are more than $M - 1$ neighbors of agent $i$ within sensing radius $R$, which, from the definition of $M$, implies that the MAS is unsafe.

\noindent\textbf{Case 3: } $\abs{\tilde{\mathcal{N}}_i} = M$, i.e., the number of neighbors is $M$. In this case, if a neighbor is added to $\tilde{\mathcal{N}}_i$ without any other agent leaving the set $\tilde{\mathcal N}_i$, then we obtain Case 2 and the MAS is unsafe. If a node $j$ is removed with no other neighbors added to $\tilde{\mathcal{N}}_i$, then this happens when $j$ leaves the sensing radius $R$. Finally, if a node $j$ is removed at the same time that a node $k$ is added without changing the size of $\abs{\mathcal{N}_i} = M$, by the continuity of the position dynamics there exists a time $t$ where both nodes are at the same distance $R$ from $p_i$. However, this implies that $\abs{\mathcal{N}_i} = M + 1$, and the MAS is unsafe by Case 2.

Consequently, changes in the neighboring indices $\mathcal{N}_i$ can only occur without collision at a distance $R$.

{
\section{Proof of Lemma \ref{lemma:cts}}\label{app: h_cts_proof}
\begin{proof}
For convenience, define $\tilde{h} : \mathbb{R} \to \mathbb{R}$ as
\begin{equation}
    \tilde{h}(t) \coloneqq h( \bar{x}_{\mathcal{N}_i(t)}(t) ).
\end{equation}
In order to prove continuous differentiability of $\tilde h$ at each $t$, consider two cases, namely time instants when the neighborhood $\mathcal N_i(t)$ changes due to additional or removal of a neighbor node, and time instants when there is no change in $\mathcal N_i(t)$. 
First, consider $t = t_0$ such that the neighborhood set $\mathcal{N}_i(t)$ does not change at $t_0$, i.e., $\lim_{t\uparrow t_0}\mathcal N_i(t) = \lim_{t\downarrow t_0}\mathcal N_t(t) = \mathcal N_i(t_0)$. Since $h$ and $\bar{x}_{\mathcal{N}_i}$ are continuously differentiable at $t = t_0$, we obtain that $\tilde{h}$ is continuously differentiable at $t = t_0$.

Now, consider the case when $\mathcal{N}_i$ changes at $t = t_0$. For the sake of brevity, denote the neighborhood before the change as $\mathcal{N}_i^-$, and the neighborhood after the change as $\mathcal{N}_i^+$, i.e., $\mathcal N_i^-  = \lim_{t\uparrow t_0}\mathcal N_i(t)$ and $\mathcal N_i^+ = \lim_{t\downarrow t_0}\mathcal N_i(t)$.  Using $2)$ from \Cref{def: gcbf}, $\tilde{h}$ is continuous at $t = t_0$. Moreover, using 2) from \Cref{def: gcbf} on $t < t_0$ and $t \geq t_0$ separately yields the one-sided limits
\begin{equation}
    \lim_{t \uparrow t_0} \frac{h( \bar{x}_{\mathcal{N}_i^-}(t)) - h( \bar{x}_{\mathcal{N}_i^+}(t_0))}{t - t_0} = 0,
\end{equation}
and
\begin{equation}
    \lim_{t \downarrow t_0} \frac{h( \bar{x}_{\mathcal{N}_i^+}(t)) - h( \bar{x}_{\mathcal{N}_i^+}(t_0) )}{t - t_0} = 0.
\end{equation}
This, along with $1)$ from \Cref{def: gcbf} (i.e., the gradient of $h$ is zero at $t_0$) implies that the derivative of $\tilde{h}$ exists at $t_0$ and is continuous. Hence, $\tilde h$ is continuously differentiable for all $t$. 
\end{proof}
}

\section{Proof of Theorem \ref{thm: safety result}}\label{app: safety proof}

\begin{proof}
    Consider any $i \in V_a$. We first prove that $\bar{x}_{\mathcal{N}_i}(t) \in {\mathcal{B}_h}$ for all $t \geq 0$. Define $t_k$ for $k \in \mathbb{N}$ with $t_0=0$, such that $\mathcal{N}_i$ is constant on the time segments $[t_k, t_{k+1})$ for all $i$\footnote{A Zeno behavior is not possible because of the smoothness of the control input.} i.e., 
    \begin{equation}
        t_{k} \coloneqq \inf \{ t > t_{k-1} \mid \exists i : \mathcal{N}_i(t) \not= \mathcal{N}_i(t_k) \},\quad k \geq 1.
    \end{equation}
    For each such interval $[t_k, t_{k+1})$, suppose that $h( \bar{x}_{\mathcal{N}_i}(t_k ) ) \geq 0$. Then, using \cite[Lemma 3.4]{khalil2002nonlinear}
on the function $t \mapsto h( \bar{x}_{\mathcal{N}_i}(t) )$ along with \eqref{eq:graph CBF} implies that
    \begin{equation} \label{eq:h_interval_positive}
        h( \bar{x}_{\mathcal{N}_i}(t) ) \geq 0, \quad \forall t \in [t_k, t_{k+1}).
    \end{equation}
    Since this holds for all $i \in V_a$, we obtain that
    \begin{equation*}
        \bar{x}(t_k) \in \mathcal{C}_N \subset \mathcal{S}_N \implies \bar{x}(t) \in \mathcal{C}_N \subset \mathcal{S}_N, \quad \forall t \in [t_k, t_{k+1}).
    \end{equation*}
    {
    Since $t \mapsto h(\bar{x}_{\mathcal{N}_i}(t))$ is continuous from \Cref{lemma:cts}, the limit of this map at $t_{k+1}$ exists.
    Taking this limit on the left side of \eqref{eq:h_interval_positive} implies that  $h(\bar{x}_{\mathcal{N}_i(t_{k+1})}) \geq 0$.
    }

    Since $h( \bar{x}_{\mathcal{N}_i}(0) ) \geq 0$, applying induction thus gives the result that $h( \bar{x}_{\mathcal{N}_i}(t) ) \geq 0$, and hence $\bar{x}_{\mathcal{N}_i}(t) \in \mathcal{C}_i$, for all $t \geq 0$. Since this holds for all $i \in V_a$, by definition of $\mathcal{C}_N$, we have that $\bar{x} \in \mathcal{C}_N \subset \mathcal{S}_N$ for all $t \geq 0$.
\end{proof}

\section{Experiment details and more results}

\subsection{Environment details}\label{app: env details}
Here, we provide the details of each experiment environment. 

\textbf{SingleIntegrator} ~ We use single integrator dynamics as the base environment to verify the correctness of the baseline methods and to show the performance of the methods when there are no control input limits. The dynamics is given as $\dot x_i = v_i$, where $x_i=[p^x_i, p^y_i]^\top\in \mathbb R^2$ is the position of the $i$-th agent and $v_i =[v^x_i, v^y_i]^\top$ its velocity.  In this environment, we use $e_{ij}=x_j-x_i$ as the edge information.
We use the following reward function weights for training InforMARL.
\begin{equation}
    \lambda_{\text{nom}} = 0.1, \quad \lambda_{\text{goal}} = 0.1, \quad
    \lambda_{\text{col}} = 5.0\,.
\end{equation}

\textbf{DoubleIntegrator}~ We use double integrator dynamics for this environment. The state of agent $i$ is given by $x_i=[p^x_i,p^y_i,v^x_i,v^y_i]^\top$, where $[p^x_i,p^y_i]^\top$ is the position of the agent, and $[v^x_i,v^y_i]^\top$ is the velocity. The action of agent $i$ is given by $u_i=[a^x_i,a^y_i]^\top$, i.e., the acceleration. The dynamics function is given by:
\begin{equation}
    \dot x_i = \left[v^x_i, v^y_i, a^x_i, a^y_i\right]^\top
\end{equation}
The simulation time step is $\delta t=0.03$. In this environment, we use $e_{ij}=x_j-x_i$ as the edge information.
We use the following reward function weights for training InforMARL without obstacles.
\begin{equation}
    \lambda_{\text{nom}} = 0.1, \quad \lambda_{\text{goal}} = 0.1, \quad
    \lambda_{\text{col}} = 2.0\,.
\end{equation}
and the following reward function weights with obstacles.
\begin{equation}
    \lambda_{\text{nom}} = 0.1, \quad \lambda_{\text{goal}} = 0.1, \quad
    \lambda_{\text{col}} = 5.0\,.
\end{equation}
For the hand-crafted CBFs, we use $\alpha_0=10$. 

\begin{figure*}[t]
    \centering
    \includegraphics[width=\textwidth]{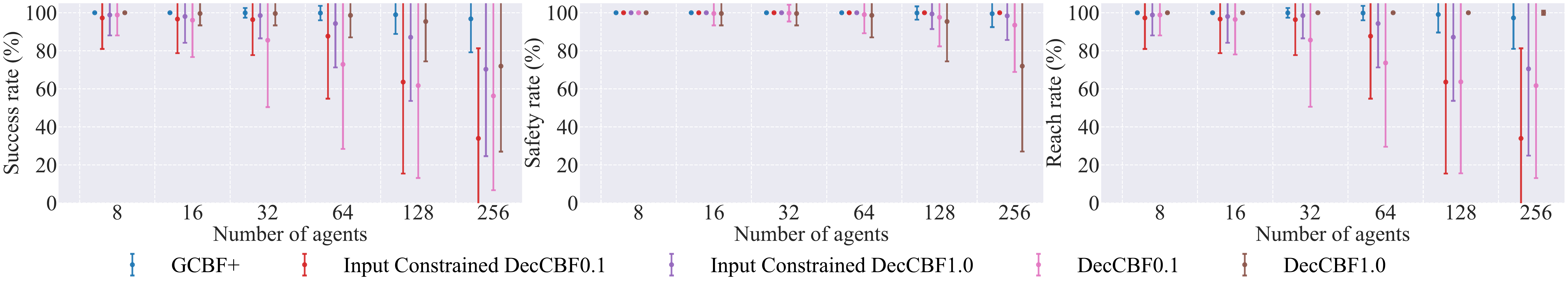}
    \caption{\textbf{Comparison with a CBF that accounts for input constraints}: 
    Success (left), safety (middle), and reach (right) rates for the proposed method (GCBF+), the decentralized CBF, and the decentralized input-constrained CBF in the DoubleIntegrator environment.}
    \label{fig: cbf-u-constrained}
\end{figure*}

\textbf{DubinsCar}~ We use the standard Dubin's car model in this environment. The state of agent $i$ is given by $x_i=[p^x_i,p^y_i,\theta_i,v_i]^\top$, where $[p^x_i,p^y_i]^\top$ is the position of the agent, $\theta_i$ is the heading, and $v_i$ is the speed. The action of agent $i$ is given by $u_i=[\omega_i,a_i]^\top$ containing angular velocity and longitudinal acceleration. The dynamics function is given by:
\begin{equation}
    \dot x_i = \left[v_i \cos(\theta_i), v_i \sin(\theta_i), \omega_i, a_i\right]^\top
\end{equation}
The simulation time step is $\delta t=0.03$. We use $e_{ij}=e_j(x_j)-e_i(x_i)$ as the edge information, where $e_i(x_i)=[p^x_i,p^y_i,v_i \cos(\theta_i), v_i \sin(\theta_i)]^\top$. 
We use the following reward function weights for training InforMARL.
\begin{equation}
    \lambda_{\text{nom}} = 0.1, \quad \lambda_{\text{goal}} = 0.1, \quad
    \lambda_{\text{col}} = 1.0\,.
\end{equation}
For the hand-crafted CBFs, we use $\alpha_0=5$. 

\textbf{LinearDrone}~ We use a linearized model for drones in our experiments. The state of agent $i$ is given by $x_i=[p^x_i,p^y_i,p^z_i,v^x_i,v^y_i,v^z_i]^\top$ where $[p^x_i,p^y_i,p^z_i]^\top$ is the 3D position, and $[v^x_i,v^y_i,v^z_i]^\top$ is the 3D velocity. The control inputs are $u_i=[a^x_i,a^y_i,a^z_i]^\top$, and the dynamics function is given by:
\begin{equation}
    \dot x_i = \left[\begin{array}{cc}
         v^x_i \\
         v^y_i \\
         v^z_i \\
         -1.1v^x_i + 1.1a^x_i \\
         -1.1v^y_i + 1.1a^y_i \\
         -6v^z_i + 6a^z_i
    \end{array}\right]
\end{equation}
The simulation time step is $\delta t=0.03$. We use $e_{ij}=x_j - x_i$ as the edge information.
We use the following reward function weights for training InforMARL.
\begin{equation}
    \lambda_{\text{nom}} = 1.0, \quad \lambda_{\text{goal}} = 0.1, \quad
    \lambda_{\text{col}} = 1.0\,.
\end{equation}
For the hand-crafted CBFs, we use $\alpha_0=3$. 

\textbf{CrazyflieDrone}~ 
\revision{One of the advantages of GCBF is that it is model-agnostic. To show that GCBF also works for other more realistic dynamics, we test GCBF for CrazyFlie dynamics}. The 6-DOF quadrotor dynamics are given in \cite{budaciu2019evaluation} with $x\in \mathbb R^{12}$ consisting of positions, velocities, angular positions and angular velocities, and $u\in \mathbb R^4$ consisting of the thrust at each of four motors:
\begin{subequations}
\begin{align}
    \dot p_x = & ~ \big(c(\phi)c(\psi)s(\theta)+s(\phi)s(\psi)\big)w\\
    & -\big(s(\psi)c(\phi)-c(\psi)s(\phi)s(\theta)\big)v + u c(\psi)c(\theta) \\
    \dot p_y = & ~ \big(s(\phi)s(\psi)s(\theta)+c(\phi)c(\psi)\big)v\\
    & -\big(c(\psi)s(\phi) -s(\psi)c(\phi)s(\theta)\big)w + u s(\psi)c(\theta)\\
    \dot p_z = & ~ w~c(\phi)c(\theta)-u~s(\theta)+v~s(\phi)c(\theta)\\
    \dot u = & ~ r~v-q~w+g~s(\theta)\\
    \dot v = & ~ p~w - r~u -g~s(\phi)c(\theta)\\
    \dot w = & ~ q~u-p~v + \frac{U_1}{m}-g~c(\theta)c(\phi)\\
    \dot \phi = & ~ r\frac{c(\phi)}{c(\theta)} + q\frac{s(\phi)}{c(\theta)} \\
    \dot \theta = & ~ q~c(\phi)-r~s(\phi)\\
    \dot \psi = & ~ p +r~c(\phi)t(\theta) + q~s(\phi)t(\theta) \\
    \dot r = & ~ \frac{1}{I_{zz}}\big(U_2-pq(I_{yy}-I_{xx})\big)\\
    \dot q = & ~ \frac{1}{I_{yy}}\big(U_3-pr(I_{xx}-I_{zz}) \big)\\
    \dot p  = & ~ \frac{1}{I_{xx}}\Big(U_4-qr(I_{zz}-I_{yy}) \Big)
\end{align}
\end{subequations}
where $m, I_{xx}, I_{yy}, I_{zz}, k_r, k_t>0$ are system parameters, $g = 9.8$ is the gravitational acceleration, $c(\cdot), s(\cdot), t(\cdot)$ denote $\cos(\cdot), \sin(\cdot), \tan(\cdot)$, respectively,  $(p_x, p_y, p_z)$ denote the position of the quadrotor, $(\phi, \theta, \psi)$ its Euler angles and $u = (U_1, U_2, U_3, U_4)$ the input vector consisting of thrust $U_1$ and moments $U_2, U_3, U_4$.\footnote{We noticed that the quadrotor dynamics in \cite{budaciu2019evaluation} (as well as the references they cite) has a couple of typos (in $\dot p_z$ and $\dot p$). We have fixed those typos using first principles.}

The relation between the vector $u$ and the individual motor speeds is given as {\small 
\begin{align}
    \begin{bmatrix}U_1\\ U_2 \\ U_3\\ U_4 \end{bmatrix}\!=\! \begin{bmatrix}C_T & C_T& C_T& C_T\\
    -dC_T\sqrt{2} & -dC_T\sqrt{2} & dC_T\sqrt{2} & dC_T\sqrt{2}\\
    -dC_T\sqrt{2} & dC_T\sqrt{2} & dC_T\sqrt{2} & -dC_T\sqrt{2}\\
    -C_D & C_D  & -C_D & C_D
    \end{bmatrix} \!\begin{bmatrix}\omega_1^2\\ \omega_2^2\\\omega_3^2\\\omega_4^2\end{bmatrix}\!,
\end{align}}\normalsize
where $\omega_i$ is the angular speed of the $i-$th motor for $i\in \{1, 2, 3, 4\}$, $C_D$ is the drag coefficient and $C_T$ is the thrust coefficient. These parameters are given as: $I_{xx} = I_{yy} = 1.395\times 10^{-5}$ kg-$\textnormal{m}^2$, $I_{zz} = 2,173\times 10^{-5}$ kg-$\textnormal{m}^2$, $m = 0.0299$ kg, $C_T = 3.1582\times 10^{-10}$ N/rpm$^2$, $C_D = 7.9379\times 10^{-12}$ N/rpm$^2$ and $d = 0.03973$ m (see \cite{budaciu2019evaluation}). 
We use the following reward function weights for training InforMARL.
\begin{equation}
    \lambda_{\text{nom}} = 0.5, \quad \lambda_{\text{goal}} = 0.1, \quad
    \lambda_{\text{col}} = 2.0\,.
\end{equation}
For the hand-crafted CBFs, we use $\alpha_0=3$. 

For the CrazyflieDrone environment, we use a two-level control architecture similar to the one used in the hardware experiments described in Section \ref{sec:hardware exps}. The chosen control algorithm computes a high-level reference velocity and yaw rate as the input, which a low-level LQR controller then tracks. 

\subsection{Comparison with CBF that accounts for input constraints}\label{app: cbf-input-constraint}
We compared GCBF+ with generic hand-crafted CBFs that do not account for input constraints in \Cref{sec:experiments}, as there is no systematic way of hand-crafting a CBF for multi-agent systems with input constraints. However, \cite{wang2017safety} proposes a CBF-based \textit{hierarchical} control framework that satisfies input constraints for double integrator systems. The pair-wise CBF is given by
\begin{equation}\label{eq: inp cont hc cbf}
    h_{ij} = \sqrt{4u_\mathrm{max}(\|p_i-p_j\|-2r)} + \frac{(p_i - p_j)^\top}{\|p_i - p_j\|}(v_i - v_j),
\end{equation}
where $u_\mathrm{max}$ is the input constraint. However, this CBF candidate might not be a valid CBF because it does not necessarily satisfy the CBF condition \eqref{eq: cbf-descent-cond} and can result in an infeasible CBF-QP. To address this problem, the authors of \cite{wang2017safety} develop a hierarchical control framework where whenever the CBF-QP is infeasible for agent $i$, it decelerates with the maximum acceleration $u_\mathrm{max}$. We generate two baselines following this framework with $\alpha=1.0$ and $\alpha=0.1$ in the CBF condition \eqref{eq: hc CBF condition} and compare GCBF+ with these two baselines (Input Constrained DecCBF $\alpha$), as well as with the decentralized framework introduced in \Cref{sec: baselines} using the CBF in \eqref{eq: inp cont hc cbf}. The results are shown in \Cref{fig: cbf-u-constrained}.
We can observe that the safety rate of Input Constrained DecCBF is close to $100\%$, much higher than DecCBF due to explicitly considering input constraints.
In light of the safety guarantees in Theorem VI.2 of \cite{wang2017safety}, we suspect that the safety violations of Input Constrained DecCBF occur due to time discretization errors which are not accounted for in the proof of safety in \cite{wang2017safety}.
However, since the agents need to decelerate with maximum acceleration whenever the CBF-QP is infeasible, the reach rates of the Input Constrained DecCBF are low, resulting in a low success rate.
GCBF+, however, considers input constraints during the learning process, and the learned GCBF does not need another hierarchical controller to maintain safety. Moreover, this hybrid controller of \cite{wang2017safety} only works with double integrator dynamics and cannot be directly used with other dynamics.

\begin{figure}[t]
    \centering
    \includegraphics[width=.8\columnwidth]{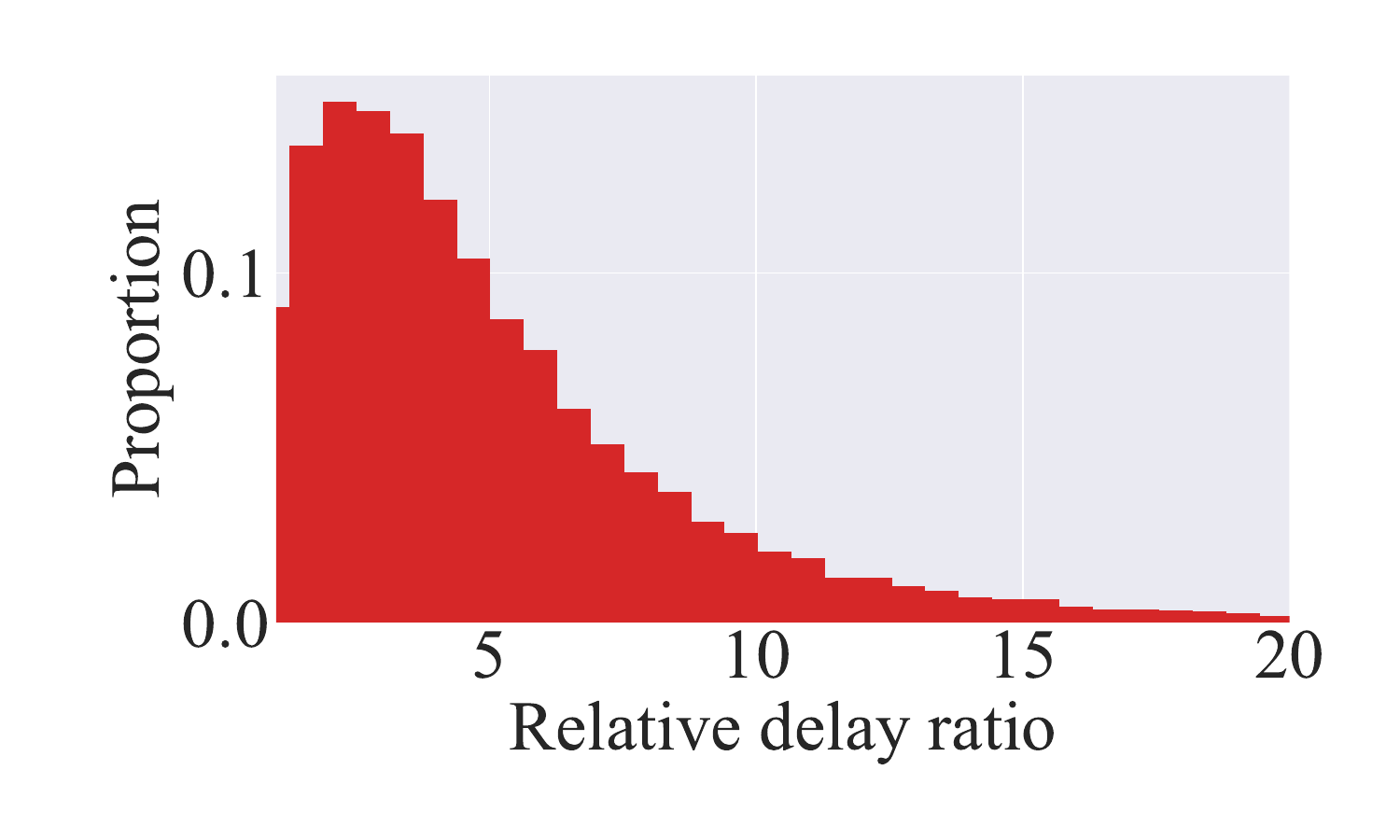}
    \caption{Histogram of the relative delay ratio, defined as the timestep taken by GCBF+ to reach the goal divided by the timestep taken by the nominal controller to reach the goal in the DoubleIntegrator environment.}
    \label{fig: reach-time}
\end{figure}

\subsection{Analysis of time delay}
Apart from the goal-reaching rate, we report the distribution of the relative delay ratio, defined as $t_{\textnormal{GCBF+}} / t_{\textnormal{nom}}$ where $t_{\textnormal{GCBF+}}$ and $t_{\textnormal{nom}}$ are the timesteps taken by the agents to reach their goals using GCBF+ and the nominal controller, respectively, in the DoubleIntegrator environment with $1024$ agents in \Cref{fig: reach-time}. We can observe that about $90\%$ of the agents reach their goals within $10\times$ the timesteps taken by the nominal controller. 

\end{appendices}

\section*{Acknowledgments}
This work was partly supported by the National Science Foundation (NSF) CAREER Award \#CCF-2238030, the MIT Lincoln Lab under the Safety in Aerobatic Flight Regimes (SAFR) program, and the MIT-DSTA program. Any opinions, findings, conclusions, or recommendations expressed in this publication are those of the authors and don’t necessarily reflect the views of the sponsors.

\bibliographystyle{IEEEtran}
\bibliography{refs}

\begin{thebibliography}{10}
\providecommand{\url}[1]{#1}
\csname url@samestyle\endcsname
\providecommand{\newblock}{\relax}
\providecommand{\bibinfo}[2]{#2}
\providecommand{\BIBentrySTDinterwordspacing}{\spaceskip=0pt\relax}
\providecommand{\BIBentryALTinterwordstretchfactor}{4}
\providecommand{\BIBentryALTinterwordspacing}{\spaceskip=\fontdimen2\font plus
\BIBentryALTinterwordstretchfactor\fontdimen3\font minus
  \fontdimen4\font\relax}
\providecommand{\BIBforeignlanguage}[2]{{%
\expandafter\ifx\csname l@#1\endcsname\relax
\typeout{** WARNING: IEEEtran.bst: No hyphenation pattern has been}%
\typeout{** loaded for the language `#1'. Using the pattern for}%
\typeout{** the default language instead.}%
\else
\language=\csname l@#1\endcsname
\fi
#2}}
\providecommand{\BIBdecl}{\relax}
\BIBdecl

\bibitem{dorri2018multi}
A.~Dorri, S.~S. Kanhere, and R.~Jurdak, ``Multi-agent systems: A survey,''
  \emph{IEEE Access}, vol.~6, pp. 28\,573--28\,593, 2018.

\bibitem{baiyu2023DD}
B.~Li and H.~Ma, ``Double-deck multi-agent pickup and delivery: Multi-robot
  rearrangement in large-scale warehouses,'' \emph{IEEE Robotics and Automation
  Letters}, vol.~8, no.~6, pp. 3701--3708, 2023.

\bibitem{kattepur2018distributed}
A.~Kattepur, H.~K. Rath, A.~Simha, and A.~Mukherjee, ``Distributed optimization
  in multi-agent robotics for industry 4.0 warehouses,'' in \emph{Proceedings
  of the 33rd Annual ACM Symposium on Applied Computing}, 2018, pp. 808--815.

\bibitem{schmidt2022introduction}
L.~M. Schmidt, J.~Brosig, A.~Plinge, B.~M. Eskofier, and C.~Mutschler, ``An
  introduction to multi-agent reinforcement learning and review of its
  application to autonomous mobility,'' in \emph{2022 IEEE 25th International
  Conference on Intelligent Transportation Systems (ITSC)}.\hskip 1em plus
  0.5em minus 0.4em\relax IEEE, 2022, pp. 1342--1349.

\bibitem{palanisamy2020multi}
P.~Palanisamy, ``Multi-agent connected autonomous driving using deep
  reinforcement learning,'' in \emph{2020 International Joint Conference on
  Neural Networks (IJCNN)}.\hskip 1em plus 0.5em minus 0.4em\relax IEEE, 2020,
  pp. 1--7.

\bibitem{zhou2021smarts}
M.~Zhou, J.~Luo, J.~Villella, Y.~Yang, D.~Rusu, J.~Miao, W.~Zhang, M.~Alban,
  I.~Fadakar, Z.~Chen \emph{et~al.}, ``Smarts: An open-source scalable
  multi-agent rl training school for autonomous driving,'' in \emph{Conference
  on Robot Learning}.\hskip 1em plus 0.5em minus 0.4em\relax PMLR, 2021, pp.
  264--285.

\bibitem{zhang2023compositional}
S.~Zhang, Y.~Xiu, G.~Qu, and C.~Fan, ``Compositional neural certificates for
  networked dynamical systems,'' in \emph{Learning for Dynamics and Control
  Conference}.\hskip 1em plus 0.5em minus 0.4em\relax PMLR, 2023, pp. 272--285.

\bibitem{tian2020search}
Y.~Tian, K.~Liu, K.~Ok, L.~Tran, D.~Allen, N.~Roy, and J.~P. How, ``Search and
  rescue under the forest canopy using multiple uavs,'' \emph{The International
  Journal of Robotics Research}, vol.~39, no. 10-11, pp. 1201--1221, 2020.

\bibitem{ghamry2017multiple}
K.~A. Ghamry, M.~A. Kamel, and Y.~Zhang, ``Multiple uavs in forest fire
  fighting mission using particle swarm optimization,'' in \emph{2017
  International Conference on Unmanned Aircraft Systems (ICUAS)}.\hskip 1em
  plus 0.5em minus 0.4em\relax IEEE, 2017, pp. 1404--1409.

\bibitem{ju2022review}
C.~Ju, J.~Kim, J.~Seol, and H.~I. Son, ``A review on multirobot systems in
  agriculture,'' \emph{Computers and Electronics in Agriculture}, vol. 202, p.
  107336, 2022.

\bibitem{chen2021scalable}
J.~Chen, J.~Li, C.~Fan, and B.~C. Williams, ``Scalable and safe multi-agent
  motion planning with nonlinear dynamics and bounded disturbances,'' in
  \emph{Proceedings of the AAAI Conference on Artificial Intelligence},
  vol.~35, 2021, pp. 11\,237--11\,245.

\bibitem{afonso2020task}
R.~J. Afonso, M.~R. Maximo, and R.~K. Galv{\~a}o, ``Task allocation and
  trajectory planning for multiple agents in the presence of obstacle and
  connectivity constraints with mixed-integer linear programming,''
  \emph{International Journal of Robust and Nonlinear Control}, vol.~30,
  no.~14, pp. 5464--5491, 2020.

\bibitem{netter2021bounded}
J.~Netter, G.~P. Kontoudis, and K.~G. Vamvoudakis, ``Bounded rational rrt-qx:
  Multi-agent motion planning in dynamic human-like environments using
  cognitive hierarchy and q-learning,'' in \emph{2021 60th IEEE Conference on
  Decision and Control (CDC)}.\hskip 1em plus 0.5em minus 0.4em\relax IEEE,
  2021, pp. 3597--3602.

\bibitem{saravanos2023ADMM}
A.~D. Saravanos, Y.~Aoyama, H.~Zhu, and E.~A. Theodorou, ``Distributed
  differential dynamic programming architectures for large-scale multiagent
  control,'' \emph{IEEE Transactions on Robotics}, vol.~39, no.~6, pp.
  4387--4407, 2023.

\bibitem{garg2023learning}
K.~Garg, S.~Zhang, O.~So, C.~Dawson, and C.~Fan, ``Learning safe control for
  multi-robot systems: Methods, verification, and open challenges,''
  \emph{Annual Reviews in Control}, vol.~57, p. 100948, 2024.

\bibitem{yu2022surprising}
C.~Yu, A.~Velu, E.~Vinitsky, J.~Gao, Y.~Wang, A.~Bayen, and Y.~Wu, ``The
  surprising effectiveness of ppo in cooperative multi-agent games,''
  \emph{Advances in Neural Information Processing Systems}, vol.~35, pp.
  24\,611--24\,624, 2022.

\bibitem{ames2019control}
A.~D. Ames, S.~Coogan, M.~Egerstedt, G.~Notomista, K.~Sreenath, and P.~Tabuada,
  ``Control barrier functions: Theory and applications,'' in \emph{2019 18th
  European Control Conference (ECC)}.\hskip 1em plus 0.5em minus 0.4em\relax
  IEEE, 2019, pp. 3420--3431.

\bibitem{glotfelter2017nonsmooth}
P.~Glotfelter, J.~Cort{\'e}s, and M.~Egerstedt, ``Nonsmooth barrier functions
  with applications to multi-robot systems,'' \emph{IEEE Control Systems
  Letters}, vol.~1, no.~2, pp. 310--315, 2017.

\bibitem{jankovic2021collision}
M.~Jankovic and M.~Santillo, ``Collision avoidance and liveness of multi-agent
  systems with cbf-based controllers,'' in \emph{2021 60th IEEE Conference on
  Decision and Control (CDC)}.\hskip 1em plus 0.5em minus 0.4em\relax IEEE,
  2021, pp. 6822--6828.

\bibitem{cheng2020safe}
R.~Cheng, M.~J. Khojasteh, A.~D. Ames, and J.~W. Burdick, ``Safe multi-agent
  interaction through robust control barrier functions with learned
  uncertainties,'' in \emph{2020 59th IEEE Conference on Decision and Control
  (CDC)}.\hskip 1em plus 0.5em minus 0.4em\relax IEEE, 2020, pp. 777--783.

\bibitem{garg2021robust}
K.~Garg and D.~Panagou, ``Robust control barrier and control lyapunov functions
  with fixed-time convergence guarantees,'' in \emph{2021 American Control
  Conference (ACC)}.\hskip 1em plus 0.5em minus 0.4em\relax IEEE, 2021, pp.
  2292--2297.

\bibitem{ames2017control}
A.~D. Ames, X.~Xu, J.~W. Grizzle, and P.~Tabuada, ``Control barrier function
  based quadratic programs for safety critical systems,'' \emph{IEEE
  Transactions on Automatic Control}, vol.~62, no.~8, pp. 3861--3876, 2017.

\bibitem{wang2017safety}
L.~Wang, A.~D. Ames, and M.~Egerstedt, ``Safety barrier certificates for
  collisions-free multirobot systems,'' \emph{IEEE Transactions on Robotics},
  vol.~33, no.~3, pp. 661--674, 2017.

\bibitem{chen2020guaranteed}
Y.~Chen, A.~Singletary, and A.~D. Ames, ``Guaranteed obstacle avoidance for
  multi-robot operations with limited actuation: A control barrier function
  approach,'' \emph{IEEE Control Systems Letters}, vol.~5, no.~1, pp. 127--132,
  2020.

\bibitem{agrawal2021safe}
D.~R. Agrawal and D.~Panagou, ``Safe control synthesis via input constrained
  control barrier functions,'' in \emph{2021 60th IEEE Conference on Decision
  and Control (CDC)}.\hskip 1em plus 0.5em minus 0.4em\relax IEEE, 2021, pp.
  6113--6118.

\bibitem{zhang2023distributed}
S.~Zhang, K.~Garg, and C.~Fan, ``Neural graph control barrier functions guided
  distributed collision-avoidance multi-agent control,'' in \emph{7th Annual
  Conference on Robot Learning}, 2023.

\bibitem{nayak2023scalable}
S.~Nayak, K.~Choi, W.~Ding, S.~Dolan, K.~Gopalakrishnan, and H.~Balakrishnan,
  ``Scalable multi-agent reinforcement learning through intelligent information
  aggregation,'' in \emph{International Conference on Machine Learning}.\hskip
  1em plus 0.5em minus 0.4em\relax PMLR, 2023, pp. 25\,817--25\,833.

\bibitem{sathya2018embedded}
A.~Sathya, P.~Sopasakis, R.~Van~Parys, A.~Themelis, G.~Pipeleers, and
  P.~Patrinos, ``Embedded nonlinear model predictive control for obstacle
  avoidance using panoc,'' in \emph{2018 European control conference
  (ECC)}.\hskip 1em plus 0.5em minus 0.4em\relax IEEE, 2018, pp. 1523--1528.

\bibitem{ma2019searching}
H.~Ma, D.~Harabor, P.~J. Stuckey, J.~Li, and S.~Koenig, ``Searching with
  consistent prioritization for multi-agent path finding,'' in
  \emph{Proceedings of the AAAI Conference on Artificial Intelligence},
  vol.~33, no.~01, 2019, pp. 7643--7650.

\bibitem{sharon2015conflict}
G.~Sharon, R.~Stern, A.~Felner, and N.~R. Sturtevant, ``Conflict-based search
  for optimal multi-agent pathfinding,'' \emph{Artificial Intelligence}, vol.
  219, pp. 40--66, 2015.

\bibitem{arul2021v}
S.~H. Arul and D.~Manocha, ``V-rvo: Decentralized multi-agent collision
  avoidance using voronoi diagrams and reciprocal velocity obstacles,'' in
  \emph{2021 IEEE/RSJ International Conference on Intelligent Robots and
  Systems (IROS)}.\hskip 1em plus 0.5em minus 0.4em\relax IEEE, 2021, pp.
  8097--8104.

\bibitem{zheng2018magent}
L.~Zheng, J.~Yang, H.~Cai, M.~Zhou, W.~Zhang, J.~Wang, and Y.~Yu, ``Magent: A
  many-agent reinforcement learning platform for artificial collective
  intelligence,'' in \emph{Proceedings of the AAAI Conference on Artificial
  Intelligence}, vol.~32, no.~1, 2018.

\bibitem{wang2014synthesis}
P.~Wang and B.~Ding, ``A synthesis approach of distributed model predictive
  control for homogeneous multi-agent system with collision avoidance,''
  \emph{International Journal of Control}, vol.~87, no.~1, pp. 52--63, 2014.

\bibitem{toumieh2022decentralized}
C.~Toumieh and A.~Lambert, ``Decentralized multi-agent planning using model
  predictive control and time-aware safe corridors,'' \emph{IEEE Robotics and
  Automation Letters}, vol.~7, no.~4, pp. 11\,110--11\,117, 2022.

\bibitem{zhu2020trajectory}
E.~L. Zhu, Y.~R. St{\"u}rz, U.~Rosolia, and F.~Borrelli, ``Trajectory
  optimization for nonlinear multi-agent systems using decentralized learning
  model predictive control,'' in \emph{2020 59th IEEE Conference on Decision
  and Control (CDC)}.\hskip 1em plus 0.5em minus 0.4em\relax IEEE, 2020, pp.
  6198--6203.

\bibitem{fedele2023distributed}
G.~Fedele and G.~Franz{\`e}, ``A distributed model predictive control strategy
  for constrained multi-agent systems: The uncertain target capturing
  scenario,'' \emph{IEEE Transactions on Automation Science and Engineering},
  2023.

\bibitem{luis2019trajectory}
C.~E. Luis and A.~P. Schoellig, ``Trajectory generation for multiagent
  point-to-point transitions via distributed model predictive control,''
  \emph{IEEE Robotics and Automation Letters}, vol.~4, no.~2, pp. 375--382,
  2019.

\bibitem{conte2012computational}
C.~Conte, T.~Summers, M.~N. Zeilinger, M.~Morari, and C.~N. Jones,
  ``Computational aspects of distributed optimization in model predictive
  control,'' in \emph{2012 IEEE 51st IEEE conference on decision and control
  (CDC)}.\hskip 1em plus 0.5em minus 0.4em\relax IEEE, 2012, pp. 6819--6824.

\bibitem{nedic2018distributed}
A.~Nedi{\'c} and J.~Liu, ``Distributed optimization for control,'' \emph{Annual
  Review of Control, Robotics, and Autonomous Systems}, vol.~1, pp. 77--103,
  2018.

\bibitem{mestres2023distributed}
P.~Mestres and J.~Cort{\'e}s, ``Distributed and anytime algorithm for network
  optimization problems with separable structure,'' in \emph{2023 62nd IEEE
  Conference on Decision and Control (CDC)}.\hskip 1em plus 0.5em minus
  0.4em\relax IEEE, 2023, pp. 5463--5468.

\bibitem{prajna2002introducing}
S.~Prajna, A.~Papachristodoulou, and P.~A. Parrilo, ``Introducing sostools: A
  general purpose sum of squares programming solver,'' in \emph{Proceedings of
  the 41st IEEE Conference on Decision and Control, 2002.}, vol.~1.\hskip 1em
  plus 0.5em minus 0.4em\relax IEEE, 2002, pp. 741--746.

\bibitem{xu2017correctness}
X.~Xu, J.~W. Grizzle, P.~Tabuada, and A.~D. Ames, ``Correctness guarantees for
  the composition of lane keeping and adaptive cruise control,'' \emph{IEEE
  Transactions on Automation Science and Engineering}, vol.~15, no.~3, pp.
  1216--1229, 2017.

\bibitem{srinivasan2021extent}
M.~Srinivasan, M.~Abate, G.~Nilsson, and S.~Coogan, ``Extent-compatible control
  barrier functions,'' \emph{Systems \& Control Letters}, vol. 150, p. 104895,
  2021.

\bibitem{zhao2023convex}
P.~Zhao, R.~Ghabcheloo, Y.~Cheng, H.~Abdi, and N.~Hovakimyan, ``Convex
  synthesis of control barrier functions under input constraints,'' \emph{IEEE
  Control Systems Letters}, 2023.

\bibitem{ahmadi2016some}
A.~A. Ahmadi and A.~Majumdar, ``Some applications of polynomial optimization in
  operations research and real-time decision making,'' \emph{Optimization
  Letters}, vol.~10, pp. 709--729, 2016.

\bibitem{cai2021safe}
Z.~Cai, H.~Cao, W.~Lu, L.~Zhang, and H.~Xiong, ``Safe multi-agent reinforcement
  learning through decentralized multiple control barrier functions,''
  \emph{arXiv preprint arXiv:2103.12553}, 2021.

\bibitem{qin2021learning}
\BIBentryALTinterwordspacing
Z.~Qin, K.~Zhang, Y.~Chen, J.~Chen, and C.~Fan, ``Learning safe multi-agent
  control with decentralized neural barrier certificates,'' in
  \emph{International Conference on Learning Representations}, 2021. [Online].
  Available: \url{https://openreview.net/forum?id=P6_q1BRxY8Q}
\BIBentrySTDinterwordspacing

\bibitem{Fernandez-Ayala2023distributed}
V.~N. Fernandez-Ayala, X.~Tan, and D.~V. Dimarogonas, ``Distributed barrier
  function-enabled human-in-the-loop control for multi-robot systems,'' in
  \emph{2023 IEEE International Conference on Robotics and Automation (ICRA)},
  2023, pp. 7706--7712.

\bibitem{wang2024distributed}
H.~Wang, A.~Papachristodoulou, and K.~Margellos, ``Distributed control design
  and safety verification for multi-agent systems.''\hskip 1em plus 0.5em minus
  0.4em\relax IEEE, 2024, pp. 5481--5486.

\bibitem{dawson2023safe}
C.~Dawson, S.~Gao, and C.~Fan, ``Safe control with learned certificates: A
  survey of neural lyapunov, barrier, and contraction methods for robotics and
  control,'' \emph{IEEE Transactions on Robotics}, 2023.

\bibitem{dawson2022safe}
C.~Dawson, Z.~Qin, S.~Gao, and C.~Fan, ``Safe nonlinear control using robust
  neural lyapunov-barrier functions,'' in \emph{Conference on Robot
  Learning}.\hskip 1em plus 0.5em minus 0.4em\relax PMLR, 2022, pp. 1724--1735.

\bibitem{qin2022sablas}
Z.~Qin, D.~Sun, and C.~Fan, ``Sablas: Learning safe control for black-box
  dynamical systems,'' \emph{IEEE Robotics and Automation Letters}, vol.~7,
  no.~2, pp. 1928--1935, 2022.

\bibitem{so2023train}
O.~So, Z.~Serlin, M.~Mann, J.~Gonzales, K.~Rutledge, N.~Roy, and C.~Fan, ``How
  to train your neural control barrier function: Learning safety filters for
  complex input-constrained systems,'' in \emph{2024 IEEE International
  Conference on Robotics and Automation (ICRA)}.\hskip 1em plus 0.5em minus
  0.4em\relax IEEE, 2024, pp. 11\,532--11\,539.

\bibitem{meng2021reactive}
Y.~Meng, Z.~Qin, and C.~Fan, ``Reactive and safe road user simulations using
  neural barrier certificates,'' in \emph{2021 IEEE/RSJ International
  Conference on Intelligent Robots and Systems (IROS)}.\hskip 1em plus 0.5em
  minus 0.4em\relax IEEE, 2021, pp. 6299--6306.

\bibitem{dinneweth2022multi}
J.~Dinneweth, A.~Boubezoul, R.~Mandiau, and S.~Espi{\'e}, ``Multi-agent
  reinforcement learning for autonomous vehicles: a survey,'' \emph{Autonomous
  Intelligent Systems}, vol.~2, no.~1, p.~27, 2022.

\bibitem{zhang2019mamps}
W.~Zhang, O.~Bastani, and V.~Kumar, ``Mamps: Safe multi-agent reinforcement
  learning via model predictive shielding,'' \emph{arXiv preprint
  arXiv:1910.12639}, 2019.

\bibitem{qie2019joint}
H.~Qie, D.~Shi, T.~Shen, X.~Xu, Y.~Li, and L.~Wang, ``Joint optimization of
  multi-uav target assignment and path planning based on multi-agent
  reinforcement learning,'' \emph{IEEE access}, vol.~7, pp. 146\,264--146\,272,
  2019.

\bibitem{everett2018motion}
M.~Everett, Y.~F. Chen, and J.~P. How, ``Motion planning among dynamic,
  decision-making agents with deep reinforcement learning,'' in \emph{2018
  IEEE/RSJ International Conference on Intelligent Robots and Systems
  (IROS)}.\hskip 1em plus 0.5em minus 0.4em\relax IEEE, 2018, pp. 3052--3059.

\bibitem{xiao2022motion}
X.~Xiao, B.~Liu, G.~Warnell, and P.~Stone, ``Motion planning and control for
  mobile robot navigation using machine learning: a survey,'' \emph{Autonomous
  Robots}, vol.~46, no.~5, pp. 569--597, 2022.

\bibitem{dai2023socially}
Z.~Dai, T.~Zhou, K.~Shao, D.~H. Mguni, B.~Wang, and H.~Jianye,
  ``Socially-attentive policy optimization in multi-agent self-driving
  system,'' in \emph{Conference on Robot Learning}.\hskip 1em plus 0.5em minus
  0.4em\relax PMLR, 2023, pp. 946--955.

\bibitem{pan2022mate}
X.~Pan, M.~Liu, F.~Zhong, Y.~Yang, S.-C. Zhu, and Y.~Wang, ``Mate: Benchmarking
  multi-agent reinforcement learning in distributed target coverage control,''
  \emph{Advances in Neural Information Processing Systems}, vol.~35, pp.
  27\,862--27\,879, 2022.

\bibitem{wang2022darl1n}
B.~Wang, J.~Xie, and N.~Atanasov, ``Darl1n: Distributed multi-agent
  reinforcement learning with one-hop neighbors,'' in \emph{2022 IEEE/RSJ
  International Conference on Intelligent Robots and Systems (IROS)}.\hskip 1em
  plus 0.5em minus 0.4em\relax IEEE, 2022, pp. 9003--9010.

\bibitem{wang2022distributed}
Y.~Wang, M.~Damani, P.~Wang, Y.~Cao, and G.~Sartoretti, ``Distributed
  reinforcement learning for robot teams: a review,'' \emph{Current Robotics
  Reports}, vol.~3, no.~4, pp. 239--257, 2022.

\bibitem{yu2023learning}
C.~Yu, H.~Yu, and S.~Gao, ``Learning control admissibility models with graph
  neural networks for multi-agent navigation,'' in \emph{Conference on Robot
  Learning}.\hskip 1em plus 0.5em minus 0.4em\relax PMLR, 2023, pp. 934--945.

\bibitem{blumenkamp2022framework}
J.~Blumenkamp, S.~Morad, J.~Gielis, Q.~Li, and A.~Prorok, ``A framework for
  real-world multi-robot systems running decentralized gnn-based policies,'' in
  \emph{2022 International Conference on Robotics and Automation (ICRA)}.\hskip
  1em plus 0.5em minus 0.4em\relax IEEE, 2022, pp. 8772--8778.

\bibitem{jia2022multi}
X.~Jia, L.~Sun, H.~Zhao, M.~Tomizuka, and W.~Zhan, ``Multi-agent trajectory
  prediction by combining egocentric and allocentric views,'' in
  \emph{Conference on Robot Learning}.\hskip 1em plus 0.5em minus 0.4em\relax
  PMLR, 2022, pp. 1434--1443.

\bibitem{tolstaya2021multi}
E.~Tolstaya, J.~Paulos, V.~Kumar, and A.~Ribeiro, ``Multi-robot coverage and
  exploration using spatial graph neural networks,'' in \emph{2021 IEEE/RSJ
  International Conference on Intelligent Robots and Systems (IROS)}.\hskip 1em
  plus 0.5em minus 0.4em\relax IEEE, 2021, pp. 8944--8950.

\bibitem{li2020graph}
Q.~Li, F.~Gama, A.~Ribeiro, and A.~Prorok, ``Graph neural networks for
  decentralized multi-robot path planning,'' in \emph{2020 IEEE/RSJ
  International Conference on Intelligent Robots and Systems (IROS)}.\hskip 1em
  plus 0.5em minus 0.4em\relax IEEE, 2020, pp. 11\,785--11\,792.

\bibitem{li2019graph}
Y.~Li, C.~Gu, T.~Dullien, O.~Vinyals, and P.~Kohli, ``Graph matching networks
  for learning the similarity of graph structured objects,'' in
  \emph{International Conference on Machine Learning}.\hskip 1em plus 0.5em
  minus 0.4em\relax PMLR, 2019, pp. 3835--3845.

\bibitem{blanchini1999set}
F.~Blanchini, ``Set invariance in control,'' \emph{Automatica}, vol.~35,
  no.~11, pp. 1747--1767, 1999.

\bibitem{pereiradecentralized}
M.~A. Pereira, A.~D. Saravanos, O.~So, and E.~A. Theodorou, ``Decentralized
  safe multi-agent stochastic optimal control using deep {FBSDE}s and {ADMM},''
  in \emph{Robotics: Science and Systems}, 2022.

\bibitem{li2015gated}
Y.~Li, D.~Tarlow, M.~Brockschmidt, and R.~Zemel, ``Gated graph sequence neural
  networks,'' \emph{arXiv preprint arXiv:1511.05493}, 2015.

\bibitem{hsu2021safety}
K.-C. Hsu, V.~Rubies-Royo, C.~Tomlin, and J.~F. Fisac, ``{Safety and Liveness
  Guarantees through Reach-Avoid Reinforcement Learning},'' in
  \emph{Proceedings of Robotics: Science and Systems}, Virtual, July 2021.

\bibitem{fisac2019bridging}
J.~F. Fisac, N.~F. Lugovoy, V.~Rubies-Royo, S.~Ghosh, and C.~J. Tomlin,
  ``Bridging hamilton-jacobi safety analysis and reinforcement learning,'' in
  \emph{2019 International Conference on Robotics and Automation (ICRA)}.\hskip
  1em plus 0.5em minus 0.4em\relax IEEE, 2019, pp. 8550--8556.

\bibitem{schafer2023scalable}
L.~Sch{\"a}fer, F.~Gruber, and M.~Althoff, ``Scalable computation of robust
  control invariant sets of nonlinear systems,'' \emph{IEEE Transactions on
  Automatic Control}, 2023.

\bibitem{kingma2014adam}
D.~P. Kingma and J.~Ba, ``Adam: A method for stochastic optimization,''
  \emph{arXiv preprint arXiv:1412.6980}, 2014.

\bibitem{semnani2020multi}
S.~H. Semnani, H.~Liu, M.~Everett, A.~De~Ruiter, and J.~P. How, ``Multi-agent
  motion planning for dense and dynamic environments via deep reinforcement
  learning,'' \emph{IEEE Robotics and Automation Letters}, vol.~5, no.~2, pp.
  3221--3226, 2020.

\bibitem{chen2017decentralized}
Y.~F. Chen, M.~Liu, M.~Everett, and J.~P. How, ``Decentralized
  non-communicating multiagent collision avoidance with deep reinforcement
  learning,'' in \emph{2017 IEEE International Conference on Robotics and
  Automation}.\hskip 1em plus 0.5em minus 0.4em\relax IEEE, 2017, pp. 285--292.

\bibitem{andersson2019casadi}
J.~A. Andersson, J.~Gillis, G.~Horn, J.~B. Rawlings, and M.~Diehl, ``Casadi: a
  software framework for nonlinear optimization and optimal control,''
  \emph{Mathematical Programming Computation}, vol.~11, pp. 1--36, 2019.

\bibitem{gill2005snopt}
P.~E. Gill, W.~Murray, and M.~A. Saunders, ``Snopt: An sqp algorithm for
  large-scale constrained optimization,'' \emph{SIAM review}, vol.~47, no.~1,
  pp. 99--131, 2005.

\bibitem{nguyen2016exponential}
Q.~Nguyen and K.~Sreenath, ``Exponential control barrier functions for
  enforcing high relative-degree safety-critical constraints,'' in \emph{2016
  American Control Conference (ACC)}.\hskip 1em plus 0.5em minus 0.4em\relax
  IEEE, 2016, pp. 322--328.

\bibitem{katz2017reluplex}
G.~Katz, C.~Barrett, D.~L. Dill, K.~Julian, and M.~J. Kochenderfer, ``Reluplex:
  An efficient smt solver for verifying deep neural networks,'' in
  \emph{Computer Aided Verification: 29th International Conference, CAV 2017,
  Heidelberg, Germany, July 24-28, 2017, Proceedings, Part I 30}.\hskip 1em
  plus 0.5em minus 0.4em\relax Springer, 2017, pp. 97--117.

\bibitem{khalil2002nonlinear}
H.~K. Khalil, ``Nonlinear systems third edition,'' \emph{Patience Hall}, vol.
  115, 2002.

\bibitem{budaciu2019evaluation}
C.~Budaciu, N.~Botezatu, M.~Kloetzer, and A.~Burlacu, ``On the evaluation of
  the crazyflie modular quadcopter system,'' in \emph{2019 24th IEEE
  International Conference on Emerging Technologies and Factory Automation
  (ETFA)}.\hskip 1em plus 0.5em minus 0.4em\relax IEEE, 2019, pp. 1189--1195.

\end{thebibliography}

\begin{biography}[{\includegraphics[width=1in,height=1.25in,clip,keepaspectratio]{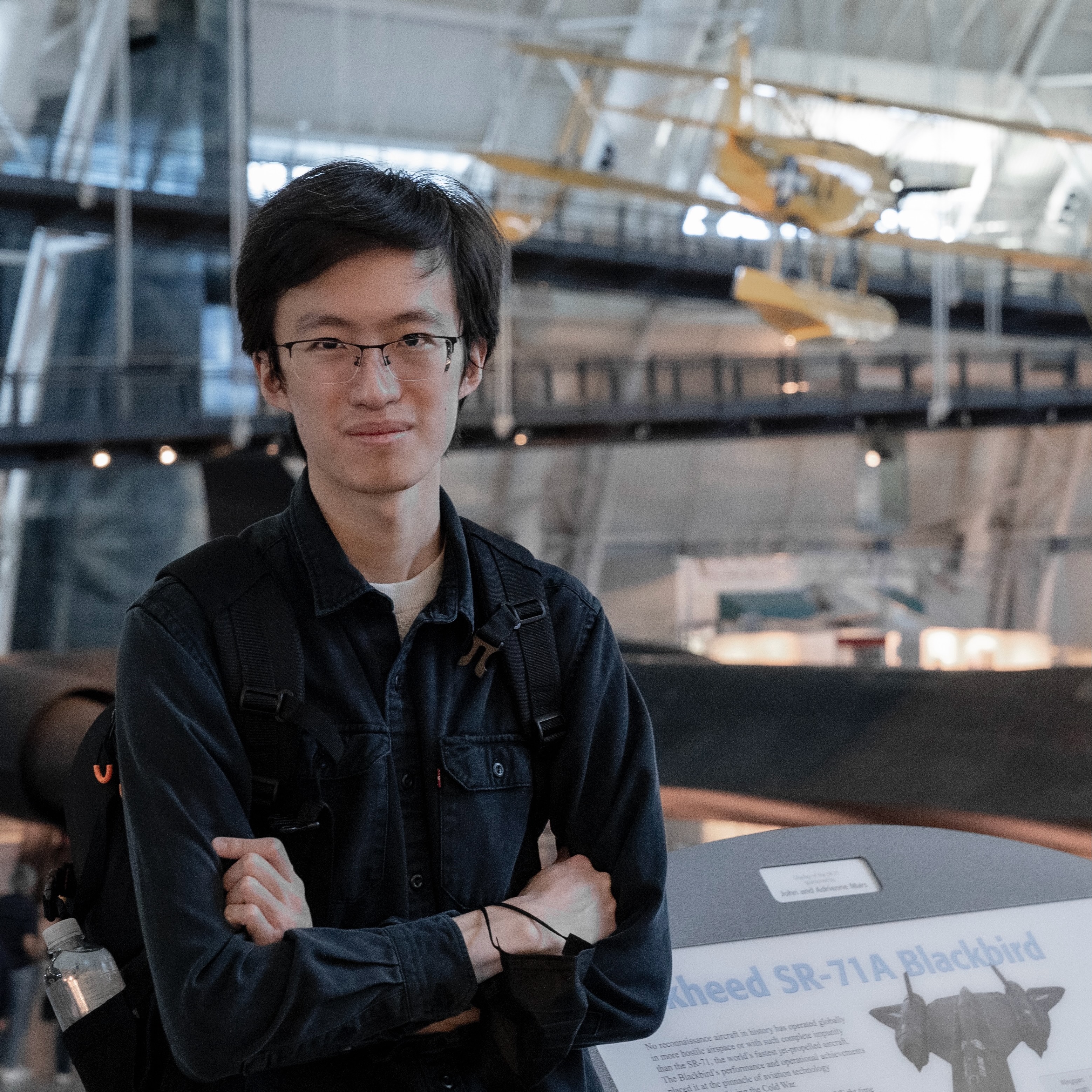}}]{Songyuan Zhang} (Student Member, IEEE) received the Bachelor of Engineering degree in engineering mechanics [Tsien Excellence in Engineering Program (TEEP)] from Tsinghua University, Beijing, China, in 2021 and the Master of Science degree in aeronautics and astronautics from the Massachusetts Institute of Technology (MIT), Cambridge, MA, USA, in 2024.

He is currently a graduate student with the Department of Aeronautics and Astronautics, MIT. His
research interests include learning safe and performant controllers for complex large-scale autonomous systems using certificates and reinforcement learning.
\end{biography}

\begin{biography}[{\includegraphics[width=1in,height=1.25in,clip,keepaspectratio]{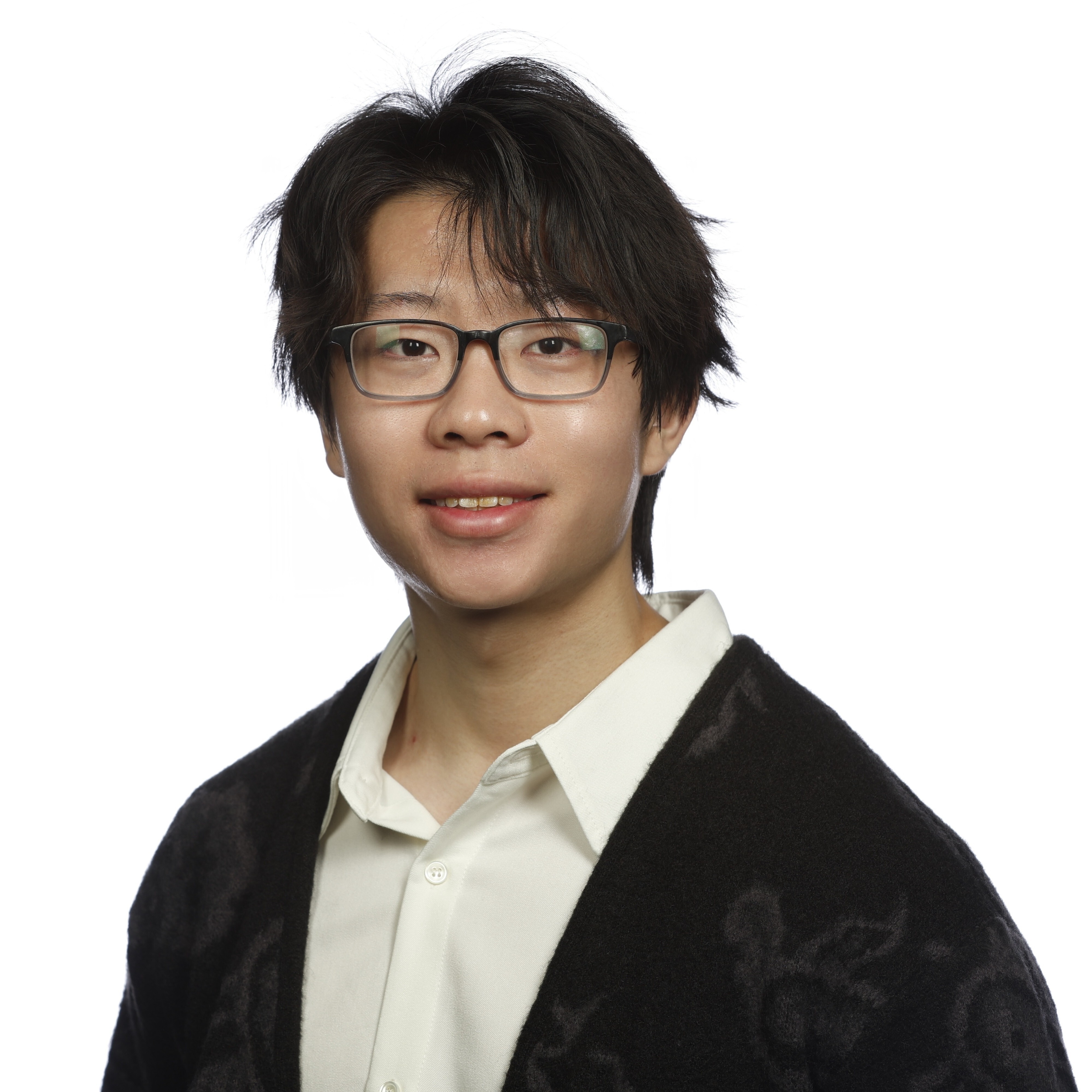}}]{Oswin So} (Graduate Student Member, IEEE) received the Bachelor of Science degree in computer science from the Georgia Institute of Technology, Atlanta, GA, Georgia, in 2021.

He is currently a graduate student with the Department of Aeronautics and Astronautics, Massachusetts Institute of Technology, Cambridge, MA, USA. His research interests include constrained optimization, safety for dynamical systems, multiagent planning, and stochastic optimal control.
\end{biography}

\begin{biography}[{\includegraphics[width=1in,height=1.25in,clip,keepaspectratio]{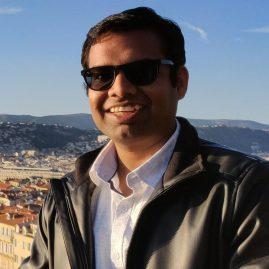}}]{Kunal Garg} (Member, IEEE) received the Master of Engineering and Ph.D. degrees in aerospace engineering from the University of Michigan, Ann Arbor, MI, USA, in 2019 and 2021, respectively, and the Bachelor of Technology degree in aerospace engineering from the Indian Institute of Technology, Mumbai, India, in 2016.

He is currently an Assistant Professor with the Mechanical and Aerospace Engineering Program, School for Engineering of Matter, Transport, and Energy, Arizona State University, Tempe, AZ, USA. Previously, he was a Postdoctoral Associate with the Department of Aeronautics and Astronautics, Massachusetts Institute of Technology, Cambridge, MA, USA, and before that, with the Department of Electrical Engineering and Computer Science, University of California, Santa Cruz, Santa Cruz, CA, USA. His research interests include robust multiagent path planning, switched and hybrid system-based analysis, and control synthesis for multiagent coordination, finite- and fixed-time stability of dynamical systems with applications to control synthesis for spatiotemporal specifications, and continuous-time optimization.
\end{biography}

\begin{biography}[{\includegraphics[width=1in,height=1.25in,clip,keepaspectratio]{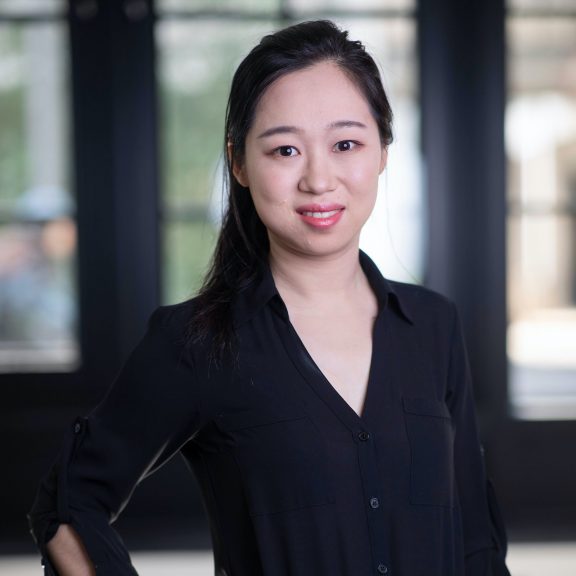}}]{Chuchu Fan} (Member, IEEE) received the Ph.D. degree in computer engineering from the Department of Electrical and Computer Engineering, University of Illinois at Urbana-Champaign, Champaign, IL, USA, in 2019.

She is currently an Associate Professor with the Department of Aeronautics and Astronautics (AeroAstro) and Laboratory for Information and Decision Systems (LIDS), Massachusetts Institute of Technology (MIT). Before that, she was a Postdoctoral Researcher with the California Institute of Technology, Pasadena, CA, USA. Her research group, Realm with MIT, works on
using rigorous mathematics, including formal methods, machine learning, and control theory, for the design, analysis, and verification of safe autonomous
systems.

Dr. Fan was the recipient of an NSF CAREER Award, an AFOSR Young Investigator Program Award, an ONR Young Investigator Program Award, and the 2020 ACM Doctoral Dissertation Award.
\end{biography}

\end{document}